\documentclass[journal,transmag]{IEEEtran}

\usepackage{url}
\usepackage{times}
\usepackage{epsfig}
\usepackage{graphicx}
\usepackage{amsmath}
\usepackage{amssymb}
\usepackage{colortbl}
\usepackage{subfigure}
\usepackage{upgreek}
\usepackage{multirow}
\usepackage{bm}

\usepackage{arydshln}
\usepackage{cite}
\usepackage{pdfpages}
\usepackage{blindtext}
\usepackage{tabularx}
\usepackage{makecell,booktabs}

\usepackage{pifont}
\newcommand{\cmark}{\ding{51}}%
\newcommand{\xmark}{\ding{55}}%
\usepackage{amsthm,xpatch}

\xpatchcmd{\proof}{\hskip\labelsep}{\hskip3\labelsep}{}{} 
\newtheorem{theorem}{Theorem}

\usepackage{caption}
\usepackage{siunitx}
\usepackage{balance}

\usepackage[pagebackref=true,breaklinks=true,letterpaper=true,colorlinks,bookmarks=false]{hyperref}
\usepackage{cleveref}
\crefformat{section}{\S#2#1#3} 
\crefformat{subsection}{\S#2#1#3}
\crefformat{subsubsection}{\S#2#1#3}

\begin{document}

\title{Scaled Simplex Representation for Subspace Clustering}

\author{
\IEEEauthorblockN{
Jun Xu$^{1,2}$,
Mengyang Yu$^{2}$,
Ling Shao$^{2}$,~\IEEEmembership{Senior Member,~IEEE},
Wangmeng Zuo$^{3}$,~\IEEEmembership{Senior Member,~IEEE},
\\
Deyu Meng$^{4}$,
Lei Zhang\textsuperscript{5},~\IEEEmembership{Fellow,~IEEE},
David Zhang\textsuperscript{5,6},~\IEEEmembership{Fellow,~IEEE}
}
\IEEEauthorblockA{
$^{1}$College of Computer Science, Nankai Univeristy, Tianjin, China
\\
$^{2}$Inception Institute of Artificial Intelligence, Abu Dhabi, UAE
\\
$^{3}$School of Computer Science and Technology, Harbin Institute of Technology, Harbin, China
\\
$^{4}$School of Mathematics and Statistics, Xi'an Jiaotong University, Xi'an, China
\\
$^{5}$Department of Computing,
The Hong Kong Polytechnic University, Hong Kong SAR, China
\\
$^{6}$School of Science and Engineering, The Chinese University of Hong Kong (Shenzhen), Shenzhen, China
}
\thanks{Corresponding author: Jun Xu (email: nankaimathxujun@gmail.com).}
}

\maketitle

\begin{abstract}
The self-expressive property of data points, i.e., each data point can be linearly represented by the other data points in the same subspace, has proven effective in leading subspace clustering methods.\ Most self-expressive methods usually construct a feasible affinity matrix from a coefficient matrix, obtained by solving an optimization problem.\ However, the negative entries in the coefficient matrix are forced to be positive when constructing the affinity matrix via exponentiation, absolute symmetrization, or squaring operations.\ This consequently damages the inherent correlations among the data.\ Besides, the affine constraint used in these methods is not flexible enough for practical applications.\ To overcome these problems, in this paper, we introduce a scaled simplex representation (SSR) for subspace clustering problem.\ Specifically, the non-negative constraint is used to make the coefficient matrix physically meaningful, and the coefficient vector is constrained to be summed up to a scalar $s<1$ to make it more discriminative.\ The proposed SSR based subspace clustering (SSRSC) model is reformulated as a linear equality-constrained problem, which is solved efficiently under the alternating direction method of multipliers framework.\ Experiments on benchmark datasets demonstrate that the proposed SSRSC algorithm is very efficient and outperforms state-of-the-art subspace clustering methods on accuracy.\ The code can be found at \url{https://github.com/csjunxu/SSRSC}.
\end{abstract}


\section{Introduction}

\IEEEPARstart{I}{mage} processing problems often contain high-dimensional data, whose structure typically lie in a union of low-dimensional subspaces~\cite{vidalsc}.\ Recovering these low-dimensional subspaces from the high-dimensional data can reduce the computational and memory cost of subsequent algorithms.\ To this end, many image processing tasks require to cluster high-dimensional data in a way that each cluster can be fitted by a low-dimensional subspace.\ This problem is known as subspace clustering (SC)~\cite{vidalsc}.

SC has been extensively studied over the past few decades~\cite{Ksubspaces,asc,multiplemotion,gpca,mppca,msl,ransac,alc,lsa,slbf,llmc,scc,rsim,ssccvpr,sscpami,lrricml,lrrpami,chentoc2014,fangtoc2016,wangtoc2017,pengtoc2017,wentoc2018,wentoc2019,leetoc2018,brbictoc2018,zhangtoc2019,yin2015dual,lrsc,Karavasilis,lsr,smr,sssc,bd,wutoc2016,mgr,s3c,s3ctip,rssc,nvr3,ji2017deep,alrg2018,yang2018automatic,sscomp,you2016oracle,you2017provable,You2018ECCV}.\ Most existing SC methods fall into four categories: iterative methods~\cite{Ksubspaces,asc}, algebraic methods~\cite{multiplemotion,gpca,rsim}, statistical methods~\cite{ransac,mppca,msl,alc}, and self-expressive methods~\cite{lsa,slbf,llmc,scc,ssccvpr,sscpami,lrricml,lrrpami,chentoc2014,fangtoc2016,wangtoc2017,pengtoc2017,wentoc2018,wentoc2019,leetoc2018,brbictoc2018,zhangtoc2019,yin2015dual,lrsc,Karavasilis,lsr,smr,sssc,bd,wutoc2016,mgr,s3c,s3ctip,rssc,nvr3,ji2017deep,alrg2018,yang2018automatic,sscomp,you2016oracle,you2017provable,You2018ECCV}.\ Among these methods, self-expressive ones are most widely studied due to their theoretical soundness and promising performance in real-world applications, such as motion segmentation~\cite{sscpami}, face clustering~\cite{lrrpami}, and digit clustering~\cite{sscomp}.\ Self-expressive methods usually follow a three-step framework.~\textbf{Step 1}: a coefficient matrix is obtained for the data points by solving an optimization problem.
\textbf{Step 2}: an affinity matrix is constructed from the coefficient matrix by employing exponentiation~\cite{scc}, absolute symmetrization~\cite{ssccvpr,sscpami,lrsc,lsr,sssc,smr,bd,wutoc2016,mgr,s3c,s3ctip,rssc,sscomp}, and squaring operations~\cite{lrricml,lrrpami,chentoc2014,fangtoc2016,wangtoc2017,pengtoc2017,wentoc2018,wentoc2019,brbictoc2018,zhangtoc2019,yin2015dual,nvr3,alrg2018}, etc. \textbf{Step 3}: spectral techniques~\cite{von2007tutorial} are applied to the affinity matrix and the final clusters are obtained for the data points. 

Self-expressive methods~\cite{scc,ssccvpr,sscpami,lrricml,lrrpami,chentoc2014,fangtoc2016,wangtoc2017,pengtoc2017,wentoc2018,wentoc2019,brbictoc2018,zhangtoc2019,yin2015dual,lrsc,Karavasilis,lsr,smr,sssc,bd,mgr,s3c,s3ctip,rssc,nvr3,ji2017deep,alrg2018,yang2018automatic,sscomp,you2016oracle,you2017provable,You2018ECCV} obtain the coefficient matrix under the \textsl{self-expressiveness} property~\cite{ssccvpr}: each data point in a union of multiple subspaces can be \textsl{linearly} represented by the other data points in the same subspace.\ In real-world applications, data points often lie in a union of multiple affine rather than linear subspaces~\cite{sscpami}, and hence the affine constraint is introduced~\cite{ssccvpr,sscpami} to constrain the sum of coefficients to be 1.\ However, most self-expressive methods~\cite{scc,ssccvpr,sscpami,lrricml,lrrpami,yin2015dual,lrsc,Karavasilis,lsr,smr,sssc,bd,mgr,s3c,s3ctip,rssc,sscomp,nvr3,you2017provable,ji2017deep,alrg2018,yang2018automatic,You2018ECCV} suffer from three major drawbacks.\ First, negative coefficients cannot be explicitly avoided in these methods in \textbf{Step 1}.\ 
But it is physically problematic to reconstruct a real data point by allowing the others to ``cancel each other out'' with complex additions and subtractions~\cite{nmfnature}.\ Second, under the affine constraint, the coefficient vector is not flexible enough to handle real-world cases where the data points are often corrupted by noise or outliers.\ Third, the exponentiation, absolute symmetrization, or squaring operations in \textbf{Step 2} force the negative coefficients to be positive, thus damaging the inherent correlations among the data points.

\begin{figure*}[ht!]
\centering
\subfigure{
\begin{minipage}{0.48\textwidth} \includegraphics[width=1\textwidth]{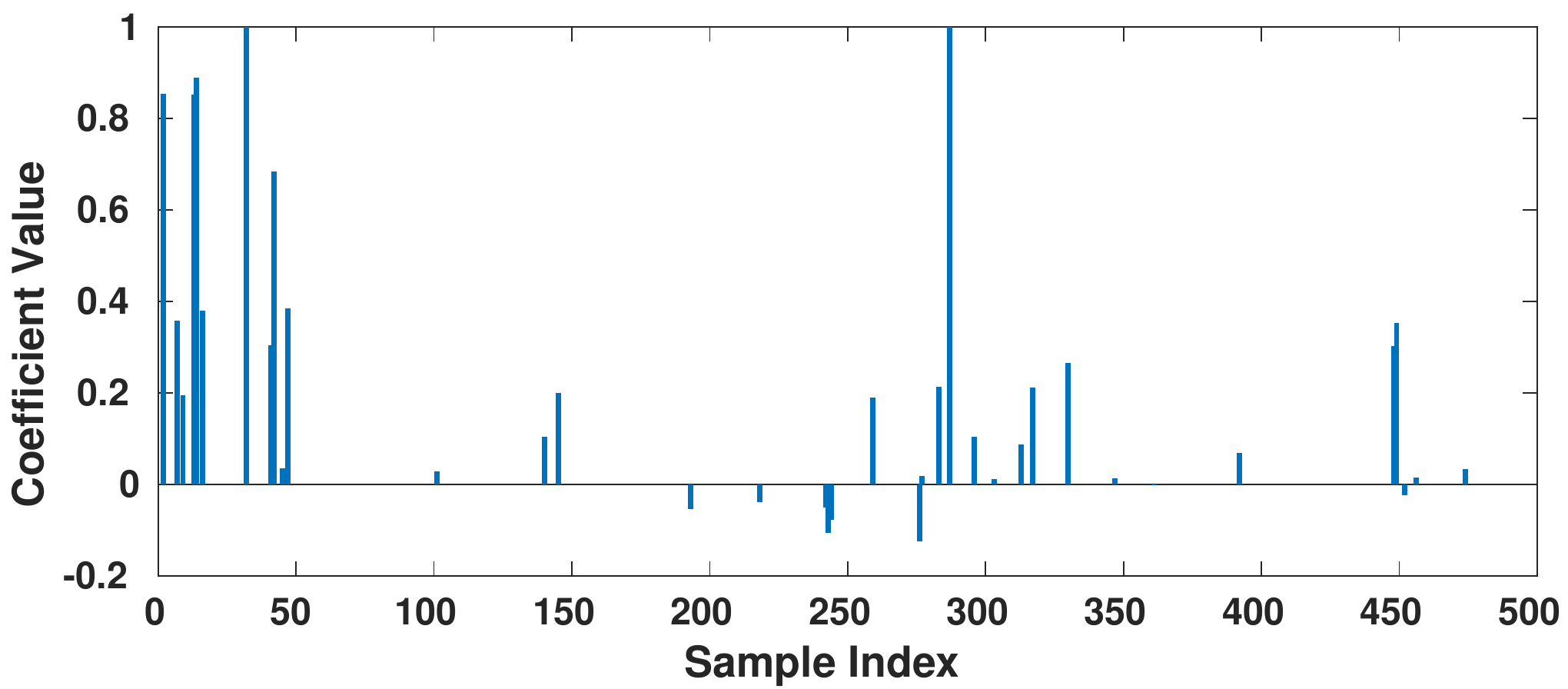}
\centering{(a) Sparse}
\end{minipage}
\begin{minipage}{0.48\textwidth} \includegraphics[width=1\textwidth]{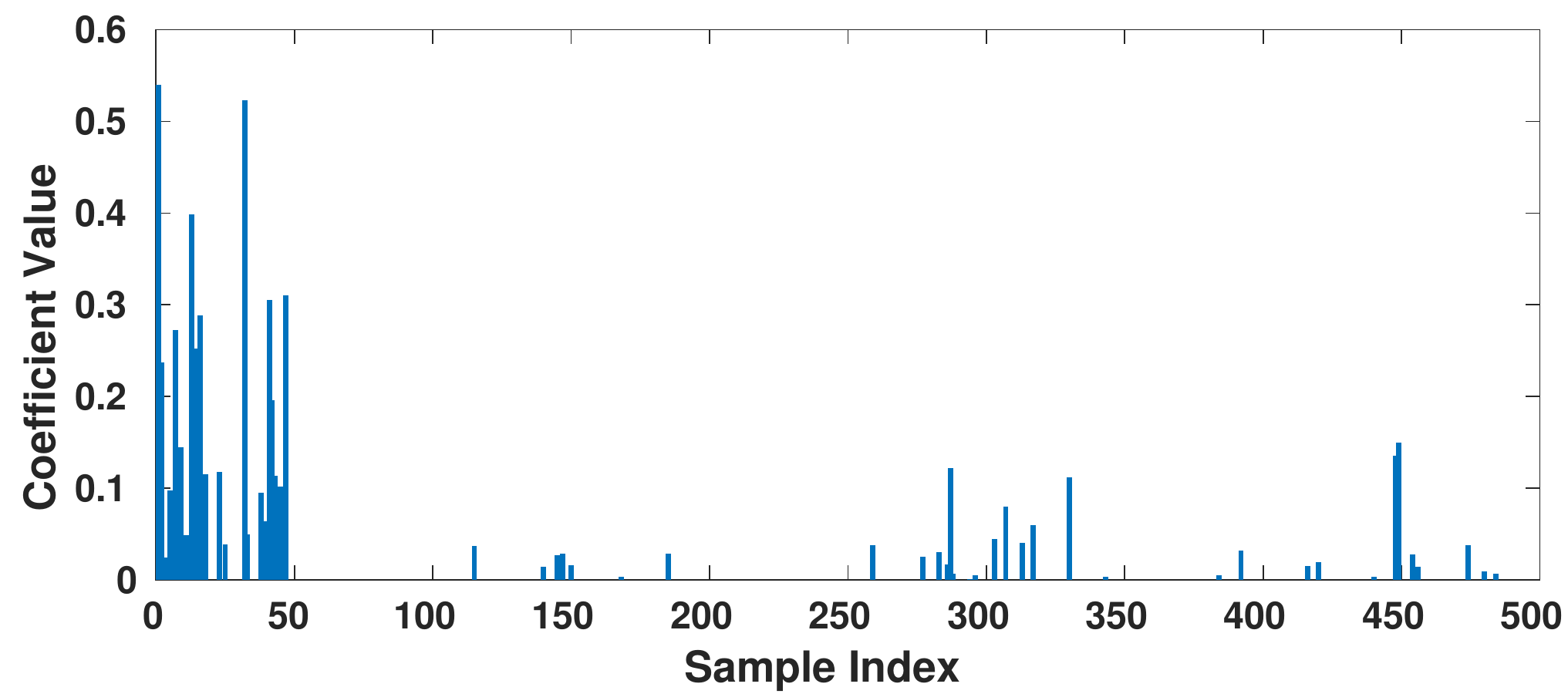}
\centering{(b) Non-negative}
\end{minipage}
}
\subfigure{
\begin{minipage}{0.48\textwidth} \includegraphics[width=1\textwidth]{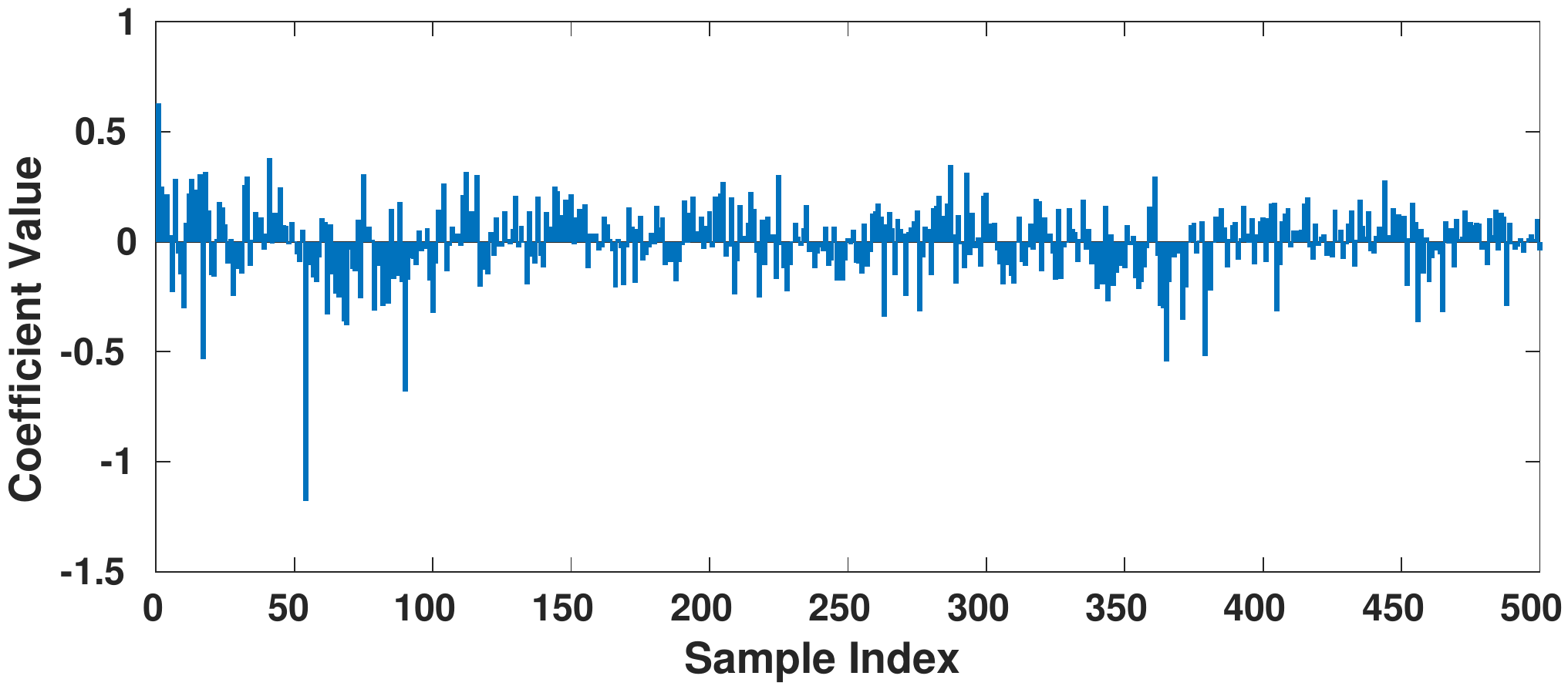}
\centering{(c) Affine}
\end{minipage}
\begin{minipage}{0.48\textwidth}  \includegraphics[width=1\textwidth]{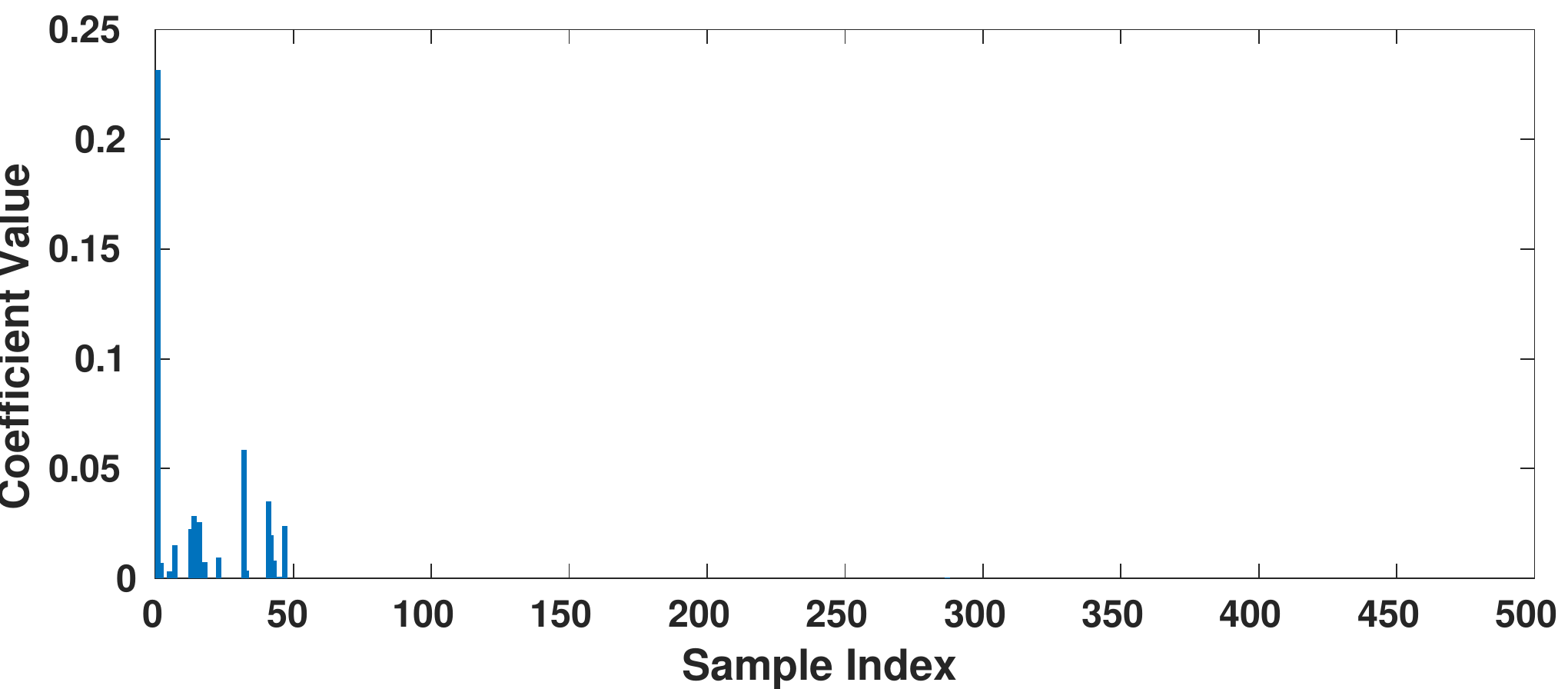}
\centering{(d) Scaled Simplex}
\end{minipage}
}
\vspace{-4mm}
\caption{Comparison of coefficient vectors by different representation schemes on the handwritten digit images from MNIST~\cite{mnist}.\ A digit sample ``0'' is used to compute the vector over 500 samples of digits $\{0,1,...,9\}$ (50 for each).\ (a) The vector solved by a sparse model (e.g., LASSO~\cite{lasso}) are confusing.\ (b) The vector solved by the least square regression (LSR) model with non-negative constraint are noisy.\ (c) The coefficients solved by LSR with affine constraint are chaotic.\ (d) The LSR with the proposed scaled simplex representation (SSR) can obtain physically more meaningful coefficients vector.}
\label{f-example}
\end{figure*}

To solve the three drawbacks mentioned above, we propose a Scaled Simplex Representation (SSR) for self-expressive based subspace clustering.\ Specifically, we first extend the affine constraint to the scaled version in the optimization model, and the coefficient vector will sum up to a scalar $s$ ($0<s<1$) instead of $1$.\ By tuning $s$, we can flexibly alter the generative and discriminative properties of the proposed SSR.\ 
Second, we utilize a non-negative constraint to make the representation more physically meaningful.\ Introducing the non-negative constraint has three primary benefits: 1) in \textbf{Step 1}, it eliminates the physically problematic subtractions among data points in optimization.\ 2) the obtained coefficient matrix in \textbf{Step 1} maintains the inherent correlations among data points by avoiding exponentiation, absolute symmetrization, or squaring operations when constructing the affinity matrix in \textbf{Step 2}.\ 3) the non-negativity could potentially enhance the discriminability of the coefficients~\cite{nnlrs} with the scaled affine constraint, such that a data point is more likely reconstructed by the data points from the same subspace.\ To illustrate the advantages of the proposed scaled simplex representation, in Figure~\ref{f-example} we show the comparison of coefficient vectors solved by different representation schemes on the handwritten digit images from MNIST~\cite{mnist}.\ A digit sample ``0'' is used to compute the vector over 500 samples of digits $\{0,1,...,9\}$ (50 for each).\ We observe that: The vector solved by a sparse model (e.g., LASSO~\cite{lasso}) are confusing, the coefficients over the samples of other digits are also non-zero (Figure~\ref{f-example} (a)).\ The vector solved by the least square regression (LSR) model with non-negative constraint are noisy (Figure~\ref{f-example} (b)).\ The coefficients solved by LSR with affine constraint are densely distributed over all samples (Figure~\ref{f-example} (c)).\ The LSR with the proposed scaled simplex representation (SSR) can obtain physically more meaningful coefficients vector (Figure~\ref{f-example} (d)).

With the introduced SSR, we propose a novel SSR-based Subspace Clustering (SSRSC) model.\ Experimental results on several benchmark datasets demonstrate that the proposed SSRSC achieves better performance than state-of-the-art algorithms.\ In summary, our contribution are three-fold:
\begin{itemize}
\item We introduce a new Scaled Simplex Representation (SSR) for subspace clustering.\ The proposed SSR can not only maintain the inherent correlations among the data points, but also handle practical problems more flexibly.

\item We propose an SSRSC model for subspace clustering.\ We reformulate the proposed SSRSC model into a linear equality-constrained problem with two variables, and solve the problem by an alternating direction method of multipliers (ADMM)~\cite{admm}.\ Each variable can be updated efficiently, and the convergence can be guaranteed. 

\item We performed comprehensive experiments on several benchmark datasets, i.e., Hopkins 155~\cite{benchmark}, ORL, Extended Yale B~\cite{YaleB}, MNIST~\cite{mnist}, and EMNIST.\ The results demonstrate that the proposed SSRSC algorithm is very efficient, and achieves better performance than state-of-the-art SC algorithms on motion segmentation, face clustering, and hand-written digits/letters clustering. 
\end{itemize}

The remainder of this paper is organized as follows. In \S\ref{sec:II}, we briefly survey the related work. In \S\ref{sec:III}, we first present the proposed SSRSC model, then provide its optimization, and finally, present the proposed SSRSC based SC algorithm.\ Extensive experiments are conducted in \S\ref{sec:IV} to evaluate the SSRSC algorithm and compare it with state-of-the-art SC algorithms.\ Conclusion and future work are given in \S\ref{sec:V}.

\section{Related Work}
\label{sec:II}
\subsection{Prior Work on Subspace Clustering}
\label{sec:II-A}
According to the employed mathematical framework, most existing SC algorithms~\cite{Ksubspaces,asc,multiplemotion,gpca,rsim,mppca,msl,ransac,alc,lsa,slbf,llmc,scc,ssccvpr,sscpami,lrricml,lrrpami,yin2015dual,lrsc,Karavasilis,Ellis,lsr,smr,sssc,bd,mgr,s3c,s3ctip,rssc,sscomp,nvr3,you2017provable,ji2017deep,alrg2018,yang2018automatic,You2018ECCV} can be divided into four main categories: iterative methods, algebraic methods, statistical methods, and self-expressive methods.~To make this section as compact as possible, please refer to~\cite{sscpami} for the introduction of iterative methods, algebraic methods, and statistical methods.\ Self-expressive methods are closely related to our work.\ They usually use local information around each data point to compare the similarity of two points. Inspired by compressed sensing theory~\cite{uncertainty,donohocs}, Sparse Subspace Clustering (SSC)~\cite{ssccvpr} solves the clustering problem by seeking a sparse representation of data points over themselves. By resolving the sparse representations for all data points and constructing an affinity graph, SSC automatically finds different subspaces as well as their dimensions from a union of subspaces.\ A robust version of SSC that deals with noise, corruptions, and missing observations is given in~\cite{sscpami}.\
Instead of finding a sparse representation, the Low-Rank Representation (LRR) method~\cite{lrricml,lrrpami,yin2015dual} poses the SC problem as finding a low-rank representation of the data over themselves. Lu \emph{et al.} proposed a clustering method based on Least Squares Regression (LSR)~\cite{lsr} to take advantage of data correlations and group highly correlated data together. The grouping information can be used to construct an affinity matrix that is block diagonal and can be used for SC through spectral clustering algorithms. Recently, Lin \emph{et al.} analyzed the grouping effect in depth and proposed a SMooth Representation (SMR) framework~\cite{smr} which also achieves a state-of-the-art performance for the subspace clustering problem. Different from SSC, the LRR, LSR, and SMR algorithms all use Normalized Cuts~\cite{shi2000normalized} in the spectral clustering step. You \emph{et al.} proposed a scalable Orthogonal Matching Pursuit (OMP) method~\cite{sscomp} to solve the SSC model~\cite{sscpami}. You \emph{et al.} also developed an elastic-net subspace clustering (EnSC) model~\cite{you2016oracle} which correctly predicted relevant connections among different clusters. Ji \emph{et al.}~\cite{ji2017deep} developed the first unsupervised network for subspace clustering by learning the \textsl{self-expressiveness} property~\cite{ssccvpr}.

\subsection{The Self-expressiveness Based Framework}
\label{sec:II-B}
Most state-of-the-art subspace clustering (SC) methods are designed under the self-expressive framework. Mathematically, denoting the data matrix as $\bm{X}=[\bm{x}_{1},...,\bm{x}_{N}]\in\mathbb{R}^{D\times N}$, each data point $\bm{x}_{i}\in\mathbb{R}^{D}$, $i\in\{1,...,N\}$ in $\bm{X}$ can be expressed as 
\vspace{-0mm}
\begin{equation}
\vspace{-0mm}
\label{e1}
\bm{x}_{i}
=
\bm{X}
\bm{c}_{i}
,
\end{equation} 
where $\bm{c}_{i}\in\mathbb{R}^{N}$ is the coefficient vector. If the data points are listed column by column, (\ref{e1}) can be rewritten as 
\begin{equation}
\label{e2}
\bm{X}
=
\bm{X}
\bm{C}
,
\end{equation} 
where $\bm{C}\in\mathbb{R}^{N\times N}$ is the coefficient matrix. To find the desired $\bm{C}$, existing SC methods \cite{ssccvpr,sscpami,lrricml,lrrpami,lrsc,lsr,sssc,bd,smr,mgr,s3c,s3ctip,rssc,sscomp,nvr3} impose various regularizations, such as sparsity and low rankness. Below, $\|\cdot\|_{F}$, $\|\cdot\|_{1}$, $\|\cdot\|_{2}$, $\|\cdot\|_{2,1}$, $\|\cdot\|_{*}$, $\lambda$, and $p$ denote the Frobenius norm, the $\ell_{1}$ norm, the $\ell_{2}$ norm, the $\ell_{2,1}$ norm, the nuclear norm, the regularization parameter, and a positive integer, respectively.\ The optimization models of several representative works are summarized as follows:

Sparse Subspace Clustering (SSC) \cite{sscpami}:
\begin{equation}
\label{e3}
\min_{\bm{C}}
\|\bm{C}\|_{1}
\quad
\text{s.t.}
\quad
\bm{X}
=
\bm{X}
\bm{C}
,
\bm{1}^{\top}\bm{C}=\bm{1}^{\top}
,
\text{diag}(\bm{C})=\bm{0}
.
\end{equation}

Low-Rank Representation (LRR) \cite{lrrpami}:
\begin{equation}
\label{e4}
\min_{\bm{C}}
\|
\bm{X}
-
\bm{X}
\bm{C}
\|
_{2,1}
+
\lambda
\|\bm{C}\|_{*}
.
\end{equation}

Least Squares Regression (LSR) \cite{lsr}:
\begin{equation}
\label{e5}
\min_{\bm{C}}
\|
\bm{X}
-
\bm{X}
\bm{C}
\|
_{F}^{2}
+
\lambda
\|\bm{C}\|_{F}^{2}
\quad
\text{s.t.}
\quad
\text{diag}(\bm{C})=\bm{0}
.
\end{equation}

SSC by Orthogonal Matching Pursuit (SSCOMP) \cite{sscomp}:
\begin{equation}
\label{e6}
\min_{\bm{c}_{i}}
\|
\bm{x}_{i}
-
\bm{X}
\bm{c}_{i}
\|
_{2}^{2}
\quad
\text{s.t.}
\quad
\|\bm{c}_{i}\|_{0}\le p
,
\bm{c}_{ii}=0
.
\end{equation}

Once the coefficient matrix $\bm{C}$ is computed, the affinity matrix $\bm{A}$ is usually constructed by exponentiation \cite{scc}, absolute symmetrization \cite{ssccvpr,sscpami,lrsc,lsr,sssc,smr,bd,mgr,s3c,s3ctip,rssc,sscomp}, and squaring operations \cite{lrricml,lrrpami,nvr3}, etc. For example, the widely used absolute symmetrization operation in \cite{ssccvpr,sscpami,lrsc,lsr,sssc,smr,bd,mgr,s3c,s3ctip,rssc,sscomp} is defined by 
\begin{equation}
\label{e7}
\bm{A}
=
(|\bm{C}|+|\bm{C}^{\top}|)/2.
\end{equation}
After the affinity matrix $\bm{A}$ is obtained, spectral clustering techniques~\cite{shi2000normalized} are applied to obtain the final segmentation of the subspaces.\ However, these self-expressive methods suffer from one major drawback: the exponentiation, absolute symmetrization, and squaring operations will force the negative entries in $\bm{C}$ to be positive in $\bm{A}$, and hence damage the inherent correlations among the data points in $\bm{X}$.\ Besides, the affine constraint in SSC limits its flexibility, making it difficult to deal with complex real world applications.\ In order to remedy these drawbacks, in this paper, we introduce the Scaled Simplex Representation to tackle the SC problem.

\subsection{Other Constraints for Subspace Clustering }
\label{sec:II-C}

Though sparse~\cite{ssccvpr,sscpami} or low-rank~\cite{lrricml,lrrpami,peng2017integrating} representation, and ridge regression~\cite{lsr,peng2015robust} are widely used in subspace clustering and other vision tasks~\cite{pgpd,mcwnnm,guide,Liang_2018_CVPR,twsc}.\ 
there are other constraints also employed by existing subspace clustering algorithms.\ In~\cite{bd}, Feng \textsl{et al.} proposed the Laplacian constraint for block-diagonal matrix pursuit.\ The developed block-diagonal SSC and LRR show clear improvements over SSC and LRR on subspace clustering and graph construction for semi-supervised learning.\ The log-determinant function is utilized in~\cite{peng2015logrank} as an alternative to the nuclear norm for low-rank constraint.\ The non-negative constraint has also been adopted in some recent work for similarity graph learning~\cite{zhuang2015constructing,kangkernel,kanglowrankkernel,kang2019clustering,xunrc2019}.\ In these work, the non-negativity is employed to construct sparse or low-rank similarity graph.\ Our work shares the same spirit with these work on this point, and aims to build physically reasonable affinity matrix by employing non-negativity.\ Recently, Li \textsl{et al.}~\cite{li2018affinity} proposed to directly learn the affinity matrix by diffusion process to spread local manifold structure of data along with its global manifold.\ This work can be viewed as propagating data manifold constraint into sparsity models to enhance its connectivity for subspace clustering.\ The simplex sparse constraint in~\cite{Huang2015ANS} is closely related to our work.\ But our proposed scaled simplex constraint allows the sum of coefficients to be a scalar, while in this work the sum is fixed to be 1.\ Besides, to make our model simpler, we do not use the diagonal constraint in~\cite{Huang2015ANS}.

\section{Simplex Representation based Subspace Clustering}
\label{sec:III}
In this section, we propose a Scaled Simplex Representation (SSR) based subspace clustering (SSRSC) model, develop an optimization algorithm to solve it, and present a novel SSRSC based algorithm for subspace clustering (SC). 

\subsection{Proposed SSRSC Model}
\label{sec:III-A}

Given a data matrix $\bm{X}$, for each data point $\bm{x}_{i}$ in $\bm{X}$, our SSRSC model aims to obtain its coefficient vector $\bm{c}_{i}$ over $\bm{X}$ under the scaled simplex constraint $\{\bm{c}_{i}\ge0,
\bm{1}^{\top}\bm{c}_{i}=s\}$.\ Here we employ the least square regression (LSR) as the objective function due to its simplicity.\ The proposed SSRSC model is formulated as follows: 
\begin{equation}
\label{e8}
\min_{\bm{c}_{i}}
\|
\bm{x}_{i}
-
\bm{X}
\bm{c}_{i}
\|
_{2}^{2}
+
\lambda
\|\bm{c}_{i}\|_{2}^{2}
\ 
\text{s.t.}
\ 
\bm{c}_{i}\ge0,
\bm{1}^{\top}\bm{c}_{i}=s
,
\end{equation}
where $\bm{1}$ is the vector of all ones and $s>0$ is a scalar denoting the sum of entries in the coefficient vector $\bm{c}_{i}$.\ We use the term ``scaled simplex'' here because the entries in the coefficient vector $\bm{c}_{i}$ are constrained by a scaled simplex, i.e., they are non-negative and sum up to a scalar $s$. 

\begin{table}[t]
\centering
\begin{tabular}{l}
\Xhline{1pt}
\textbf{Algorithm 1}: Projection of vector $\bm{u}_{k+1}$ onto a simplex
\\
\hline
\textbf{Input:} Data point $\bm{u}_{k+1}\in\mathbb{R}^{N}$, scalar $s$;
\\
1. Sort $\bm{u}_{k+1}$ into $\bm{w}$: $w_1\ge w_2\ge ...\ge w_N$;
\\
2. Find $\alpha=\max\{1\le j\le N: w_j+\frac{1}{j}(s-\sum_{i=1}^{j}w_i)>0\}$;
\\
3. Define $\beta=\frac{1}{\alpha}(s-\sum_{i=1}^{\alpha}w_i)$;
\\
\textbf{Output:} $\bm{z}_{k+1}$: ${z}_{k+1}^i=\max\{u_{k+1}^i+\beta,0\}$, $i=1,...,N$. 
\\
\Xhline{1pt}
\end{tabular}
\label{a1}
\end{table}

We can also rewrite the SSR-based model (\ref{e8}) for all $N$ data points in the matrix form:
\begin{equation}
\begin{split}
\label{e9}
\min_{\bm{C}}
&
\|
\bm{X}
-
\bm{X}
\bm{C}
\|
_{F}^{2}
+
\lambda
\|
\bm{C}
\|
_{F}^{2}
\\
&
\text{s.t.}
\quad
\bm{C}
\ge
0
,
\bm{1}^{\top}\bm{C}=s\bm{1}^{\top}
,
\end{split}
\end{equation}
where $\bm{C}\in\mathbb{R}^{N\times N}$ is the coefficient matrix. Here, the constraint $\bm{C}\ge0$ favors positive values for entries corresponding to data points from the same subspace, while suppressing entries corresponding to data points from different subspaces, thus making the coefficient matrix $\bm{C}$ discriminative. The constraint $\bm{1}^{\top}\bm{C}=s\bm{1}^{\top}$ limits the sum of each coefficient vector $\bm{c}_i$ to be $s$, thus making the representation more discriminative since each entry should be non-negative.

\subsection{Model Optimization}
\label{sec:III-B}
The proposed SSRSC model~(\ref{e9}) cannot be solved analytically. In this section, we solve it by employing variable splitting methods~\cite{courant1943,Eckstein1992}. Specifically, we introduce an auxiliary variable $\bm{Z}$ into the SSRSC model~(\ref{e9}), and reformulate it as a linear equality-constrained problem:
\begin{equation}
\begin{split}
\label{e10}
\min_{\bm{C},\bm{Z}}
&
\|
\bm{X}
-
\bm{X}\bm{C}
\|_{F}^{2}
+
\lambda
\|
\bm{Z}
\|
_{F}^{2}
\\
\text{s.t.}
\quad
&
\bm{Z}
\ge
0
,
\bm{1}^{\top}\bm{Z}=s\bm{1}^{\top}
,
\bm{Z}
=
\bm{C}
,
\end{split}
\end{equation}
whose solution w.r.t. $\bm{C}$ coincides with the solution of ($\ref{e9}$). Since the objective function in Eqn.~(\ref{e10}) is separable w.r.t. the variables $\bm{C}$ and $\bm{Z}$, it can be solved using the ADMM~\cite{admm}.\ The corresponding augmented Lagrangian function is
\begin{equation}
\begin{split}
\label{e11}
&\mathcal{L}
(\bm{C},
\bm{Z},
\bm{\Delta},
\rho)
\\
=
&
\|
\bm{X}
-
\bm{X}\bm{C}
\|_{F}^{2}
+
\lambda
\|
\bm{Z}
\|
_{F}^{2}
+
\langle
\bm{\Delta},
\bm{Z}
-
\bm{C}
\rangle
+
\frac{\rho}{2}
\|
\bm{Z}
-
\bm{C}
\|_{F}^{2}
\\
=
&
\|
\bm{X}
-
\bm{X}\bm{C}
\|_{F}^{2}
+
\lambda
\|
\bm{Z}
\|
_{F}^{2}
+
\frac{\rho}{2}
\|
\bm{Z}
-
\bm{C}
+
\frac{1}{\rho}
\bm{\Delta}
\|_{F}^{2}
\\
=
&
\|
\bm{X}
-
\bm{X}\bm{C}
\|_{F}^{2}
+
\frac{2\lambda+\rho}{2}
\|
\bm{Z}
-
\frac{\rho}{2\lambda+\rho}
(
\bm{C}
-
\frac{1}{\rho}
\bm{\Delta}
)
\|
_{F}^{2}
\\
&
+
\frac{\lambda\rho}{2\lambda+\rho}
\|
\bm{C}
-
\frac{1}{\rho}
\bm{\Delta}
\|
_{F}^{2}
,
\end{split}
\end{equation}
where $\bm{\Delta}$ is the augmented Lagrangian multiplier and $\rho>0$ is the penalty parameter.\ Denote by ($\bm{C}_{k}, \bm{Z}_{k}$) and $\bm{\Delta}_{k}$ the optimization variables and Lagrange multiplier at iteration $k$ ($k=0,1,2,...$), respectively.\ We initialize the variables $\bm{C}_{0}$, $\bm{Z}_{0}$, and $\bm{\Delta}_{0}$ to be conformable zero matrices.\ By taking derivatives of Lagrangian function $\mathcal{L}$ w.r.t. $\bm{C}$ and $\bm{Z}$, and setting them to be zeros, we can alternatively update the variables as follows:
\\
(1) \textbf{Updating $\bm{C}$ while fixing $\bm{Z}_{k}$ and $\bm{\Delta}_{k}$}:
\begin{equation}
\begin{split}
\label{e12}
\bm{C}_{k+1}
=
\arg\min_{\bm{C}}
\|
\bm{X}
-
\bm{X}\bm{C}
\|
_{F}^{2}
+
\frac{\rho}{2}
\|
\bm{C}
-
(
\bm{Z}_k
+
\frac{1}{\rho}
\bm{\Delta}_k
)
\|
_{F}^{2}
.
\end{split}
\end{equation}
This is a standard LSR problem with closed-from solution:
\begin{equation}
\begin{split}
\label{e13}
&
\hspace{-3mm}
\bm{C}_{k+1}
=
(\bm{X}^{\top}\bm{X}+\frac{\rho}{2}\bm{I})^{-1}
(\bm{X}^{\top}\bm{X}+\frac{\rho}{2}\bm{Z}_{k}+\frac{1}{2}\bm{\Delta}_{k})
.
\end{split}
\end{equation}
We note that the complexity for updating $\bm{C}$ is $\mathcal{O}(N^3)$ when there are $N$ data points in the data matrix $\bm{X}$.\ This cube complexity largely hinders the practical usage of the proposed method.\ In order to improve the speed (while maintaining the accuracy) of SSRSC, we employ the Woodbury formula~\cite{Woodbury} to compute the inversion in Eqn.~(\ref{e13}) as
\vspace{-0mm}
\begin{equation}
\vspace{-0mm}
\begin{split}
\label{e-woodbury}
\hspace{-2.5mm}
(\bm{X}^{\top}\bm{X}+\frac{\rho}{2}\bm{I})^{-1}
=
\frac{2}{\rho}\bm{I}
-
(\frac{2}{\rho})^{2}
\bm{X}^{\top}
(
\bm{I}
+
\frac{2}{\rho}\bm{X}\bm{X}^{\top}
)^{-1}
\bm{X}
.
\end{split}
\end{equation}
By this step, the complexity of updating $\bm{C}$ is reduced from $\mathcal{O}(N^3)$ to $\mathcal{O}(DN^{2})$.\ Since $(\bm{X}^{\top}\bm{X}+\frac{\rho}{2}\bm{I})^{-1}$ is not updated during iterations, we can also pre-compute it and store it before iterations.\ This further saves abundant computational costs.

\noindent
(2) \textbf{Updating $\bm{Z}$ while fixing $\bm{C}_{k}$ and $\bm{\Delta}_{k}$}:
\begin{equation}
\begin{split}
\label{e14}
\bm{Z}_{k+1}
&
=
\arg\min_{\bm{Z}}
\|
\bm{Z}
-
\frac{\rho}{2\lambda+\rho}
(
\bm{C}_{k+1}-\rho^{-1}\bm{\Delta}_{k}
)
\|_{F}^{2}
\\
&
\quad
\text{s.t.}
\quad 
\bm{Z}\ge0
,
\bm{1}^{\top}\bm{Z}=s\bm{1}^{\top}
.
\end{split}
\end{equation}
This is a quadratic programming problem with strictly convex objective function and a close and convex constraint, so there is a unique solution.\ As such, problem (\ref{e14}) can be solved using, for example, active set methods \cite{Nocedal2006,ds3} or projection based methods~\cite{Michelot1986,duchi2008efficient,Condat2016}.\ Here, we employ the projection based method~\cite{duchi2008efficient}, whose computational complexity is $\mathcal{O}(N\log{N})$ to project a vector of length $N$ onto a simplex.\ Denoting by $\bm{u}_{k+1}$ an arbitrary column of $\frac{\rho}{2\lambda+\rho}(\bm{C}_{k+1}-\rho^{-1}\bm{\Delta}_k)$, the solution of $\bm{z}_{k+1}$ (the corresponding column in $\bm{Z}_{k+1}$) can be solved by projecting $\bm{u}_{k+1}$ onto a simplex~\cite{duchi2008efficient}.\ The solution of problem (\ref{e14}) is summarized in Algorithm 1.

\noindent
(3) \textbf{Updating $\bm{\Delta}$ while fixing $\bm{C}_{k}$ and $\bm{Z}_{k}$}:
\begin{equation}
\begin{split}
\label{e15}
\bm{\Delta}_{k+1}
&
=
\bm{\Delta}_{k}
+
\rho
(\bm{Z}_{k+1}-\bm{C}_{k+1})
.
\end{split}
\end{equation}

\begin{table}[t!]
\centering
\begin{tabular}{l}
\Xhline{1pt}
\textbf{Algorithm 2}: Solve the SSRSC model (\ref{e10}) via ADMM
\\
\Xhline{0.5pt}
\textbf{Input:} Data matrix $\bm{X}$, $\text{Tol}>0$, $\rho>0$, $K$;
\\
\textbf{Initialization:} $\bm{C}_{0}=\bm{Z}_{0}=\bm{\Delta}_{0}=\bm{0}$, \text{T} = \text{False}, $k=0$;
\\
\textbf{While} (\text{T} == \text{False}) \textbf{do}
\\
1. Update $\bm{C}_{k+1}$ by Eqn.\ (\ref{e13});
\\
2. Update $\bm{Z}_{k+1}$ by Eqn.\ (\ref{e14});
\\
3. Update $\bm{\Delta}_{k+1}$ by Eqn.\ (\ref{e15});
\\
4. \textbf{if} (Convergence condition is satisfied) or ($k\ge K$)
\\
5.\quad \text{T} $\leftarrow$ \text{True};
\\
\quad \textbf{end if}
\\
\textbf{end while}
\\
\textbf{Output:} Matrices $\bm{C}$ and $\bm{Z}$.
\\
\Xhline{1pt}
\end{tabular}
\label{a2}
\end{table}

We repeat the above alternative updates until a certain convergence condition is satisfied or the number of iterations reaches a preset threshold $K$. Under the convergence condition of the ADMM algorithm, $\|\bm{C}_{k+1}-\bm{Z}_{k+1}\|_{F}\le \text{Tol}$, $\|\bm{C}_{k+1}-\bm{C}_{k}\|_{F}\le \text{Tol}$, and $\|\bm{Z}_{k+1}-\bm{Z}_{k}\|_{F}\le \text{Tol}$ must be simultaneously satisfied, where \text{Tol}$=0.01$ tolerates small errors.\ Since the objective function and constraints are both convex, the problem (\ref{e10}) solved by ADMM is guaranteed to converge at a global optimal solution.\ We summarize these updating procedures in Algorithm 2.

\textbf{Convergence Analysis}.\ The convergence of Algorithm 2 can be guaranteed since the overall objective function (\ref{e10}) is convex with a global optimal solution.\ In Figure~\ref{f-convergence}, we plot the convergence curves of the errors of $\|\mathbf{C}_{k+1}-\mathbf{Z}_{k+1}\|_{F}$, $\|\mathbf{Z}_{k+1}-\mathbf{Z}_{k}\|_{F}$, $\|\mathbf{C}_{k+1}-\mathbf{C}_{k}\|_{F}$.\ One can see that they are reduced to less than $\text{Tol}=0.01$ simultaneously in 5 iterations.
\begin{figure}[htp]
\centering    
\vspace{-3mm}
\begin{minipage}{0.48\textwidth}
\includegraphics[width=1\textwidth]{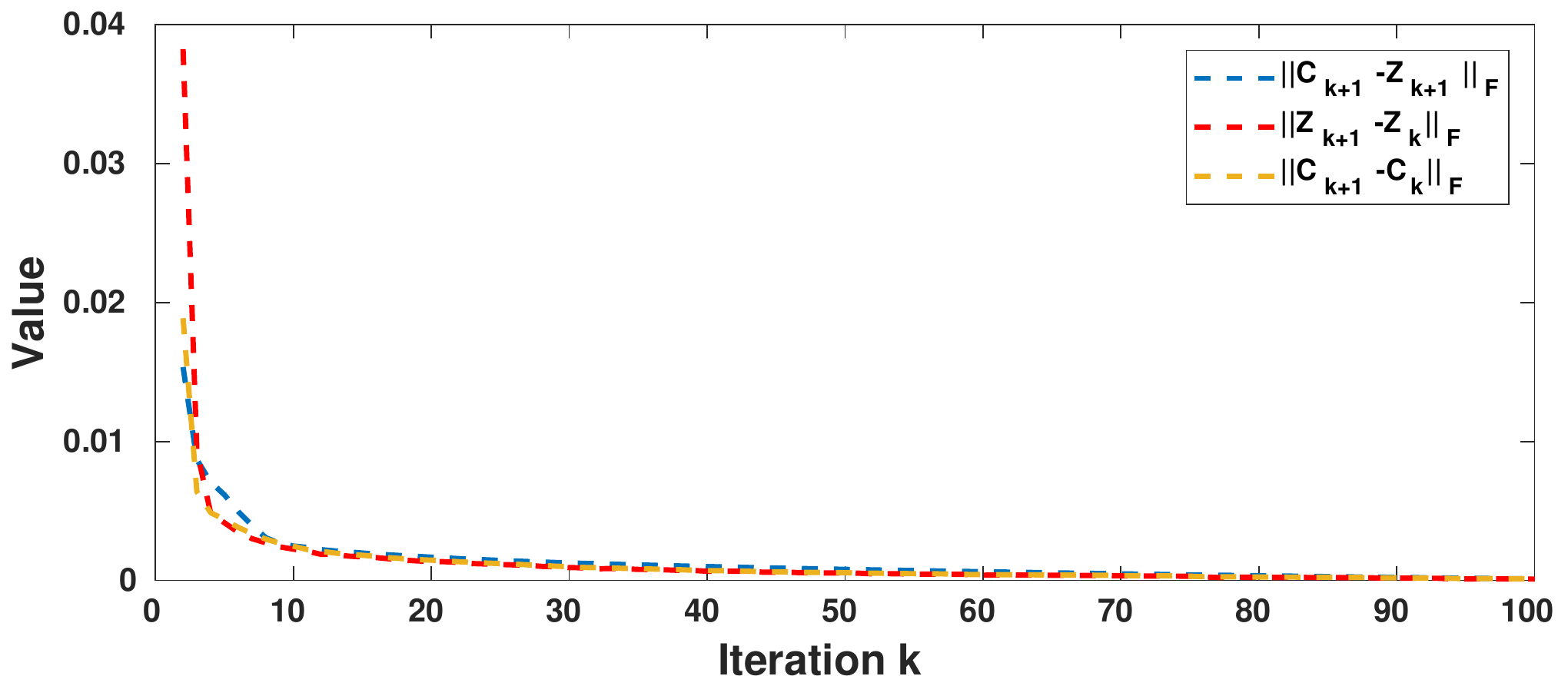}
\end{minipage}
\vspace{-2mm}
\caption{The convergence curves of $\|\mathbf{C}_{k+1}-\mathbf{Z}_{k+1}\|_{F}$ (\textcolor{red}{red} line), $\|\mathbf{Z}_{k+1}-\mathbf{Z}_{k}\|_{F}$ (\textcolor{blue}{blue} line), and $\|\mathbf{C}_{k+1}-\mathbf{C}_{k}\|_{F}$ (\textcolor{green}{green} line) of the proposed SSRSC on the ``1R2RC'' sequence from the Hopkins155 dataset~\cite{benchmark}.}
\label{f-convergence}
\vspace{-3mm}
\end{figure}


\subsection{Theoretical Analysis}
\label{sec:III-C}
Our optimization problem (\ref{e9}) is a convex optimization problem, which means $f(\bm{C}) := \|\bm{X}-\bm{X}\bm{C}\|_{F}^{2} + \lambda \|\bm{C}\|_{F}^{2}$ is a convex function and the set $\mathcal{S} := \{\bm{C}|\bm{C} \ge 0, \bm{1}^{\top}\bm{C}=s\bm{1}^{\top}\}$ is a convex set. Then, the function $f$ has a minimum on the hyper-plane $\overline{\mathcal{S}} := \{\bm{C}|\bm{1}^{\top}\bm{C}=s\bm{1}^{\top}\}$, which is the linear span of $\mathcal{S}$. Below, we discuss the solution space of $f$ over $\mathcal{S}$.

\begin{theorem}
Suppose $\bm{C}^*$ is the minimum of the convex function $f$ over the convex set $\mathcal{S}$. If $\bm{C}^*$ is not the minimum of $f$ on $\overline{\mathcal{S}}$, then $\bm{C}^*$ is on the boundary of $\mathcal{S}$: $\partial \mathcal{S}$.
\end{theorem}
\begin{proof}
We first make an assumption that $\bm{C}^* \notin \partial \mathcal{S}$, then we will derive a contradiction. If our assumption holds, there exists a high-dimensional open sphere $B_r(\bm{C}^*) = \{\bm{C} | \|\bm{C} - \bm{C}^*\| < r\}  \subset \mathbb{R}^{N \times N}$ centered at $\bm{C}^* \cap \mathcal{S}$ with radius $r > 0$, such that for all $\bm{C} \in B_r$, $f(\bm{C}) \geq f(\bm{C}^*)$ holds. 

We consider $\forall \bm{D} \in \overline{\mathcal{S}}$, $\exists \lambda \in (0,1)$ such that $\bm{C}^* + \lambda (\bm{D} - \bm{C}^*) \in B_r(\bm{C}^*)$. In fact, we can set $\lambda < \min(\frac{r}{\|\bm{D} - \bm{C}^*\|}, 1)$. Then we have  $\| \bm{C}^* + \lambda (\bm{D} - \bm{C}^*) - \bm{C}^*\| = \| \lambda (\bm{D} - \bm{C}^*) \| < r$, which means $\bm{C}^* + \lambda (\bm{D} - \bm{C}^*) \in B_r(\bm{C}^*)$. In this way, we have
\begin{equation}
\label{localmin}
f(\bm{C}^* + \lambda (\bm{D} - \bm{C}^*)) \geq f(\bm{C}^*).
\end{equation}
However, due to the convexity of $f$, we have the following Jensen's inequality:
\begin{equation}
\begin{split}
 f(\bm{C}^* + \lambda (\bm{D} - \bm{C}^*)) 
 &
 = f((1-\lambda) \bm{C}^* + \lambda \bm{D})
 \\
 &
  \leq  (1-\lambda)f(\bm{C}^*) + \lambda f(\bm{D}).
\end{split}
\end{equation}
Combining this with Eqn. (\ref{localmin}), we obtain $\lambda f(\bm{D}) \geq \lambda f(\bm{C}^*)$ and hence $f(\bm{D}) \geq f(\bm{C}^*)$, $\forall \bm{D} \in \mathbb{R}^{N \times N}$. Then, $\bm{C}^*$ is the minimum of $f$ on hyper-plane $\overline{\mathcal{S}}$, resulting a contradiction.
\end{proof}

\subsection{Discussion}
\label{sec:III-D}
The constraint ``$\bm{1}^{\top}\bm{c}=1$'' has already been used in Sparse Subspace Clustering (SSC) \cite{ssccvpr,sscpami} to deal with the presence of affine, rather than linear, subspaces. However, limiting the sum of the coefficient vector $\bm{c}$ to be $1$ is not flexible enough for real-world clustering problems. What's more, suppressing the sum of the entries in the coefficient vector $\bm{c}$ can make it more discriminative, since these entries should be non-negative and sum up to a scalar $s$. Considering the extreme case where $s$ is nearly zero, each data point must be represented by its most similar data points in the homogeneous subspace. To this end, in our proposed SSRSC model, we extend the affine constraint of ``summing up to 1'' to a scaled affine constraint of ``summing up to a scalar $s$''.\ In our experiments (please refer to \S\ref{sec:IV}), we observe improved performance of SSRSC on subspace clustering by this extension.

In real-world applications, data are often corrupted by outliers due to ad-hoc data collection techniques. Existing SC methods deal with outliers by explicitly modeling them as an additional variable, and updating this variable using the ADMM algorithm.\ For example, in the seminal work of SSC~\cite{ssccvpr,sscpami} and its successors~\cite{lsr,sssc,rssc,s3c}, $c_{ii}$ is set as $0$ for $\bm{x}_i$, indicating that each data point cannot be represented by itself, thus avoiding trivial solution of identity matrix.\ However, this brings additional computational costs and prevent the whole algorithm from converging~\cite{sscpami,lrrpami}.\ Different from these existing methods~\cite{sscpami,lsr,sssc,rssc,s3c}, we do not consider the constraint of $c_{ii}=0$, for three major reasons.\ First, the positive $\lambda$ in the regularization term can naturally prevent trivial solution of identity matrix $\bm{C}=\bm{I}$.\ Second, $c_{ii}\neq0$ has a clear physical meaning, allowing a sample $\bm{x}_{i}$ in the subspace to be \textsl{partially} represented by itself.\ This is particularly useful when $\bm{x}_i$ is corrupted by noise.\ By removing the constraint of $c_{ii}=0$, our proposed SSRSC model is more robust to noise, as will be demonstrated in the ablation study in \S\ref{sec:IV-D}.

\subsection{Subspace Clustering via Simplex Representation}
\label{sec:III-E}

\noindent
\textbf{Subspace Clustering Algorithm}.\
Denote by $\mathcal{X}=\{\bm{x}_{i}\in\mathbb{R}^{D}\}_{i=1}^{N}$ a set of data points drawn from a union of $n$ subspaces $\{\mathcal{S}_{j}\}_{j=1}^{n}$.\ Most existing spectral clustering based SC algorithms \cite{sscpami,lrrpami,lrsc,smr,s3c,sscomp} first compute the coefficient matrix $\bm{C}$, and then construct a non-negative affinity matrix $\bm{A}$ from $\bm{C}$ by exponentiation \cite{scc}, absolute symmetrization \cite{ssccvpr,sscpami,lrsc,lsr,sssc,smr,bd,mgr,s3c,rssc,sscomp}, or squaring operations \cite{lrricml,lrrpami,nvr3}, etc. In contrast, in our proposed SSRSC model, the coefficient matrix is guaranteed to be non-negative by the introduction of a simplex constraint. Hence, we can remove the absolute operation and construct the affinity matrix by
\begin{equation}
\begin{split}
\label{e18}
\bm{A}
=
(\bm{C}+\bm{C}^{\top})/2
.
\end{split}
\end{equation}

As a common post-processing step in \cite{ssccvpr,sscpami,lrsc,lsr,sssc,smr,bd,mgr,s3c,rssc,sscomp,lrricml,lrrpami,nvr3}, we apply the spectral clustering technique~\cite{ng2001spectral} to the affinity matrix $\bm{A}$, and obtain the final segmentation of data points. Specifically, we employ the widely used Normalized Cut algorithm~\cite{shi2000normalized} to segment the affinity matrix. The proposed SSRSC based subspace clustering algorithm is summarized in Algorithm 3.
\begin{table}[t!]
\centering
\begin{tabular}{l}
\Xhline{1pt}
\textbf{Algorithm 3}: Subspace Clustering by SSRSC
\\
\Xhline{0.5pt}
\textbf{Input:} A set of data points $\mathcal{X}=\{\bm{x}_{1},...,\bm{x}_{N}\}$ lying in
\\
\quad \quad \quad a union of $n$ subspaces $\{\mathcal{S}_{j}\}_{j=1}^{n}$;
\\
1. Obtain the coefficient matrix $\bm{C}$ by solving the SSRSC model:
\vspace{1mm}
\\
\qquad
$
\min_{\bm{C}}
\|
\bm{X}
-
\bm{X}
\bm{C}
\|
_{F}^{2}
+
\lambda
\|
\bm{C}
\|
_{F}^{2}
\ 
\text{s.t.}
\ 
\bm{C}
\ge
0
,
\bm{1}^{\top}\bm{C}=s\bm{1}^{\top}
$
;
\vspace{1mm}
\\
2. Construct the affinity matrix by
\vspace{1mm}
\\
\qquad\qquad
\qquad\qquad
\qquad\qquad
$
\bm{A}=\frac{\bm{C}+\bm{C}^{\top}}{2}
$
;
\vspace{1mm}
\\
3. Apply spectral clustering \cite{ng2001spectral} to the affinity matrix;
\\
\textbf{Output:} Segmentation of data: $\bm{X}_{1},...,\bm{X}_{n}$.
\\
\Xhline{1pt}
\end{tabular}
\vspace{-0mm}
\label{a3}
\end{table}

\noindent
\textbf{Complexity Analysis}.\
Assume that there are $N$ data points in the data matrix $\bm{X}$. In Algorithm 2, the costs for updating $\bm{C}$ and $\bm{Z}$ are $\mathcal{O}(N^2D)$ and $\mathcal{O}(N^2\log{N})$, respectively. The costs for updating $\bm{\Delta}$ and $\rho$ are negligible compared to the updating costs of $\bm{C}$ and $\bm{Z}$. As such, the overall complexity of Algorithm 2 is $\mathcal{O}(\text{max}(D,\log{N})N^2K)$, where $K$ is the number of iterations.\ The costs for affinity matrix construction and spectral clustering in Algorithm 3 can be ignored.\ Hence, the overall cost of the proposed SSRSC is $\mathcal{O}(\text{max}(D,\log{N})N^2K)$ for data matrix $\bm{X}\in\mathbb{R}^{D\times N}$.

\section{Experiments}
\label{sec:IV}
In this section, we first compare the proposed SSRSC with state-of-the-art subspace clustering (SC) methods.\ The comparison are performed on five benchmark datasets on motion segmentation for video analysis, human faces clustering, and hand-written digits/letters clustering.\ Then, we validate the effectiveness of the proposed scaled simplex constraints for SC through comprehensive ablation studies.

\subsection{Implementation Details} 
\label{sec:IV-A}
The proposed SSRSC model~(\ref{e9}) is solved under the ADMM \cite{admm} framework. There are five parameters to be determined in the ADMM algorithm: the regularization parameter $\lambda$, the sum $s$ of the entries in the coefficient vector, the penalty parameter $\rho$, and the iteration number $K$. In all the experiments, we fix $s=0.5$, $\rho=0.5$, and $K=5$.\ As in most competing methods~\cite{scc,ssccvpr,sscpami,lrricml,lrrpami,lrsc,lsr,smr,rsim,sssc,bd,mgr,s3c,s3ctip,rssc,sscomp,nvr3}, the parameter $\lambda$ is tuned on each dataset to achieve the best performance of SSRSC on that dataset. The influence of $\lambda$ on each dataset will be introduced in \S\ref{sec:IV-C}. All experiments are run under the Matlab2014b environment on a machine with a CPU of 3.50GHz and a 12GB RAM.

\subsection{Datasets} 
\label{sec:IV-B}
We evaluate the proposed SSRSC method on the Hopkins-155 dataset~\cite{benchmark} for motion segmentation, the Extended Yale B~\cite{YaleB} and ORL~\cite{orl} datasets for human face clustering, and the MNIST~\cite{mnist} and EMNIST~\cite{emnist} for hand-written digits/letters clustering.

\textbf{Hopkins-155} dataset~\cite{benchmark} contains 155 video sequences, 120 of which contain two moving objects and 35 of which contain three moving objects, corresponding to 2 or 3 low-dimensional subspaces of the ambient space.\ On average, each two-motion sequence has 30 frames and each frame contains $N=266$ data points, while each three-motion sequence has 29 frames and each frame contains $N=393$ data points.\ Similar to the experimental settings in previous methods, such as SSC~\cite{sscpami}, on this dataset we employ principal component analysis (PCA)~\cite{pca} to project the original trajectories of different objects into a 12-dimensional subspace, in which we evaluate the comparison methods.\ This dataset~\cite{benchmark} has been widely used as a benchmark to evaluate SC methods for motion segmentation.\ Figure~\ref{f-hopkins155} presents some segmentation examples from the Hopkins-155 dataset~\cite{benchmark}, where different colors indicate different moving objects. 
\begin{figure}[t!]
\centering
\subfigure{
\begin{minipage}[t]{0.22\textwidth}
\centering
\raisebox{-0.15cm}{\includegraphics[width=1\textwidth]{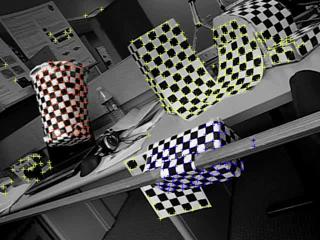}}
\end{minipage}
\begin{minipage}[t]{0.22\textwidth}
\centering
\raisebox{-0.15cm}{\includegraphics[width=1\textwidth]{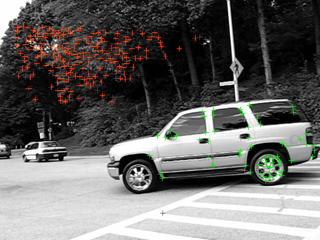}}
\end{minipage}
}
\subfigure{
\vspace{-2mm}
\begin{minipage}[t]{0.22\textwidth}
\centering
\raisebox{-0.15cm}{\includegraphics[width=1\textwidth]{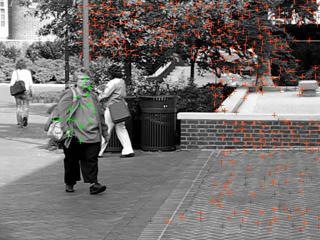}}
\end{minipage}
\begin{minipage}[t]{0.22\textwidth}
\centering
\raisebox{-0.15cm}{\includegraphics[width=1\textwidth]{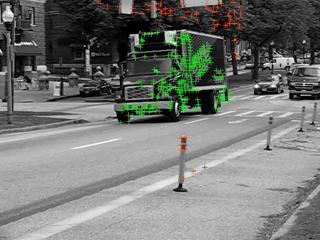}}
\end{minipage}
}
\vspace{-3mm}
\caption{Exemplar motion segmentation of one frame from different sequences in the Hopkins-155 dataset~\cite{benchmark}.}
\vspace{-3mm}
\label{f-hopkins155}
\end{figure}

\textbf{Extended Yale B} dataset~\cite{YaleB} contains face images of 38 human subjects, and each subject has 64 near-frontal images (gray-scale) taken under different illumination conditions.\ The original images are of size $192\times168$ pixels and we resize them to $48\times42$ pixels in our experiments.\ For dimension reduction purposes, the resized images are further projected onto a $6n$-dimensional subspace using PCA, where $n$ is the number of subjects (or subspaces) selected in our experiments.\ Following the experimental settings in SSC~\cite{sscpami}, we divide the 38 subjects into 4 groups, consisting of subjects 1 to 10, subjects 11 to 20, subjects 21 to 30, and subjects 31 to 38.\ For each of the first three groups we select $n\in\{2, 3, 5, 8, 10\}$ subjects, while for the last group we choose $n\in\{2, 3, 5, 8\}$.\ Finally, we apply SC algorithms for each set of $n$ subjects.\ Figure~\ref{f-yaleorl} (top) shows some face images from Extended Yale B, captured under different lighting conditions.\ 
\begin{figure}[t]
\centering
\subfigure{
\begin{minipage}[t]{0.1\textwidth}
\centering
\raisebox{-0.15cm}{\includegraphics[width=1\textwidth]{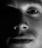}}
\end{minipage}
\begin{minipage}[t]{0.1\textwidth}
\centering
\raisebox{-0.15cm}{\includegraphics[width=1\textwidth]{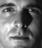}}
\end{minipage}
\begin{minipage}[t]{0.1\textwidth}
\centering
\raisebox{-0.15cm}{\includegraphics[width=1\textwidth]{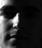}}
\end{minipage}
\begin{minipage}[t]{0.1\textwidth}
\centering
\raisebox{-0.15cm}{\includegraphics[width=1\textwidth]{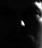}}
\end{minipage}
}
\subfigure{
\vspace{-2mm}
\begin{minipage}[t]{0.1\textwidth}
\centering
\raisebox{-0.15cm}{\includegraphics[width=1\textwidth]{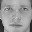}}
\end{minipage}
\begin{minipage}[t]{0.1\textwidth}
\centering
\raisebox{-0.15cm}{\includegraphics[width=1\textwidth]{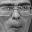}}
\end{minipage}
\begin{minipage}[t]{0.1\textwidth}
\centering
\raisebox{-0.15cm}{\includegraphics[width=1\textwidth]{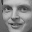}}
\end{minipage}
\begin{minipage}[t]{0.1\textwidth}
\centering
\raisebox{-0.15cm}{\includegraphics[width=1\textwidth]{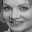}}
\end{minipage}
}
\vspace{-3mm}
\caption{Face images from the Extended Yale B dataset~\cite{YaleB} (top) and the ORL dataset~\cite{orl} (down).}
\vspace{-3mm}
\label{f-yaleorl}
\end{figure}

\textbf{ORL} dataset~\cite{orl} contains overall $400$ human face images of 40 subjects, each having 10 samples.\ Similar to~\cite{ji2017deep}, we resize the original face images from $112\times92$ to $32\times32$.\ For each human subject, the face images were taken under varying lighting conditions, with different facial expressions (e.g., open eyes or closed eyes, smiling or not smiling, etc.), as well as different facial details (e.g., w/ glasses or w/o glasses).\ The \textbf{ORL} dataset is more difficult to tackle than Extended Yale B~\cite{YaleB} for two reasons: 1) its varying face images varies much more complex; 2) it only contains 400 images, much smaller than Extended Yale B (2432 images).\ Figure~\ref{f-yaleorl} (down) shows some face images from ORL.\ 

\textbf{MNIST} dataset~\cite{mnist} contains 60,000 gray-scale images of 10 digits (i.e., $\{0,...,9\}$) in the training set and $10,000$ images in the testing set. The images are of size $28\times28$ pixels. In our experiments, we randomly select $N_{i}\in\{50,100,200,400,600\}$ images for each of the $10$ digits. Following~\cite{sscomp}, for each image, a set of feature vectors is computed using a scattering convolution network (SCN)~\cite{scn}. The final feature vector is a concatenation of the coefficients in each layer of the network, and is translation invariant and deformation stable. Each feature vector is of $3,472$-dimension. The feature vectors for all images are then projected onto a 500-dimensional subspace using PCA.\ Figure~\ref{f-mnist} shows some examples of the handwritten digit images in this dataset. 

\textbf{EMNIST} dataset~\cite{emnist} is an extension of the MNIST dataset that contains gray-scale handwritten digits and letters.\ This dataset contains overall $190,998$ images corresponding to 26 lower case letters.\ We use them as the data for a 26-class clustering problem. The images are of size $28\times28$.\ In our experiments, we randomly select $N_{i}=500$ images for each of the $26$ digits/letters. Following~\cite{You2018ECCV}, for each image, a set of feature vectors is computed using a scattering convolution network (SCN)~\cite{scn}, which is translation invariant and deformation stable.\ The feature vectors are of $3,472$-dimension, the feature vectors for all images are projected onto a 500-dimensional subspace using PCA.

\begin{figure}[t!]
\centering
\begin{minipage}[t]{0.04\textwidth}
\centering
\raisebox{-0.15cm}{\includegraphics[width=1\textwidth]{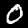}}
\end{minipage}
\begin{minipage}[t]{0.04\textwidth}
\centering
\raisebox{-0.15cm}{\includegraphics[width=1\textwidth]{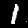}}
\end{minipage}
\begin{minipage}[t]{0.04\textwidth}
\centering
\raisebox{-0.15cm}{\includegraphics[width=1\textwidth]{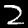}}
\end{minipage}
\begin{minipage}[t]{0.04\textwidth}
\centering
\raisebox{-0.15cm}{\includegraphics[width=1\textwidth]{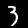}}
\end{minipage}
\begin{minipage}[t]{0.04\textwidth}
\centering
\raisebox{-0.15cm}{\includegraphics[width=1\textwidth]{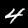}}
\end{minipage}
\begin{minipage}[t]{0.04\textwidth}
\centering
\raisebox{-0.15cm}{\includegraphics[width=1\textwidth]{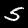}}
\end{minipage}
\begin{minipage}[t]{0.04\textwidth}
\centering
\raisebox{-0.15cm}{\includegraphics[width=1\textwidth]{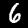}}
\end{minipage}
\begin{minipage}[t]{0.04\textwidth}
\centering
\raisebox{-0.15cm}{\includegraphics[width=1\textwidth]{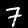}}
\end{minipage}
\begin{minipage}[t]{0.04\textwidth}
\centering
\raisebox{-0.15cm}{\includegraphics[width=1\textwidth]{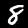}}
\end{minipage}
\begin{minipage}[t]{0.04\textwidth}
\centering
\raisebox{-0.15cm}{\includegraphics[width=1\textwidth]{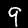}}
\end{minipage}
\vspace{-1mm}
\caption{Digit images from the MNIST dataset~\cite{mnist}.}
\vspace{-3mm}
\label{f-mnist}
\end{figure}

\subsection{Comparison with State-of-the-art Methods}
\label{sec:IV-C}
\begin{figure*}[t]
\centering
\subfigure{
\begin{minipage}[t]{0.24\textwidth}
\raisebox{-0.15cm}{\includegraphics[width=1\textwidth]{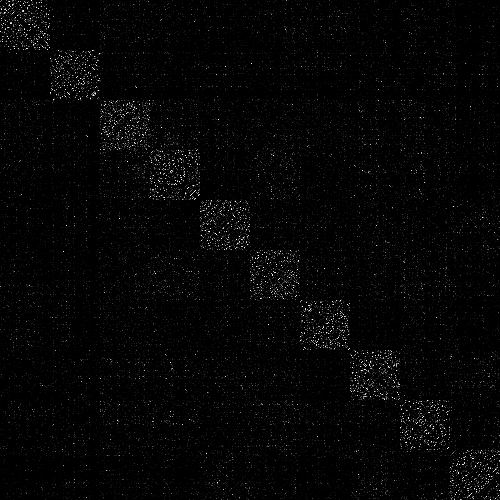}}
\centering{\small (a) SSC~\cite{sscpami}: 16.99\%}
\end{minipage}
\begin{minipage}[t]{0.24\textwidth}
\raisebox{-0.15cm}{\includegraphics[width=1\textwidth]{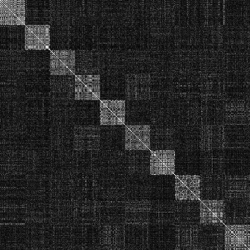}}
\centering{\small (b) LRR~\cite{lrrpami}: 17.97\%}
\end{minipage}
\begin{minipage}[t]{0.24\textwidth}
\raisebox{-0.15cm}{\includegraphics[width=1\textwidth]{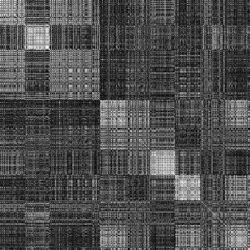}}
\centering{\small (c) LRSC~\cite{lrsc}: 24.16\%}
\end{minipage}
\begin{minipage}[t]{0.24\textwidth}
\raisebox{-0.15cm}{\includegraphics[width=1\textwidth]{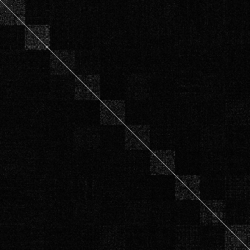}}
\centering{\small (d) LSR~\cite{lsr}: 24.98\%}
\end{minipage}
}\vspace{-3mm}
\subfigure{
\begin{minipage}[t]{0.24\textwidth}
\raisebox{-0.15cm}{\includegraphics[width=1\textwidth]{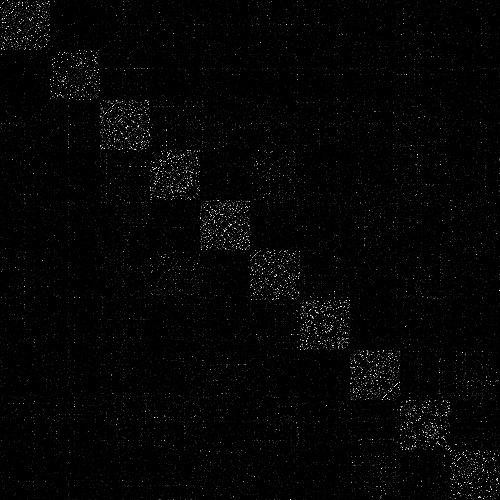}}
\centering{\small (e) S3C~\cite{s3ctip}: 15.92\%}
\end{minipage}
\begin{minipage}[t]{0.24\textwidth}
\raisebox{-0.15cm}{\includegraphics[width=1\textwidth]{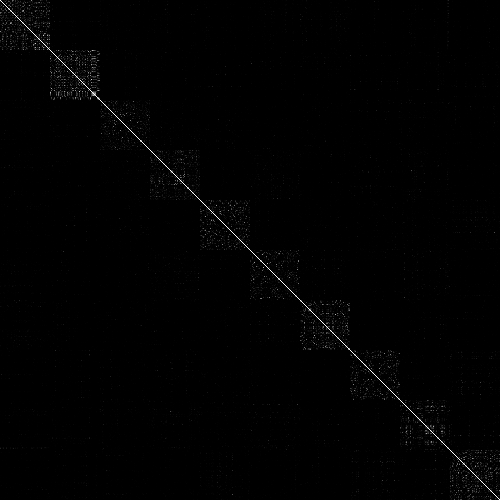}}
\centering{\small (f) RSIM~\cite{rsim}: 18.13\%}
\end{minipage}
\begin{minipage}[t]{0.24\textwidth}
\raisebox{-0.15cm}{\includegraphics[width=1\textwidth]{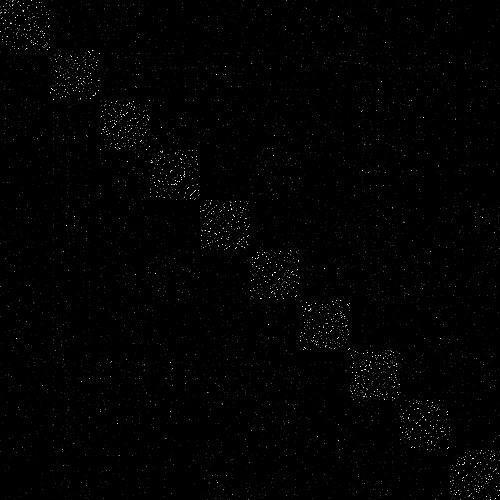}}
\centering{\small (g) SSCOMP~\cite{sscomp}: 16.36\%}
\end{minipage}
\begin{minipage}[t]{0.24\textwidth}
\raisebox{-0.15cm}{\includegraphics[width=1\textwidth]{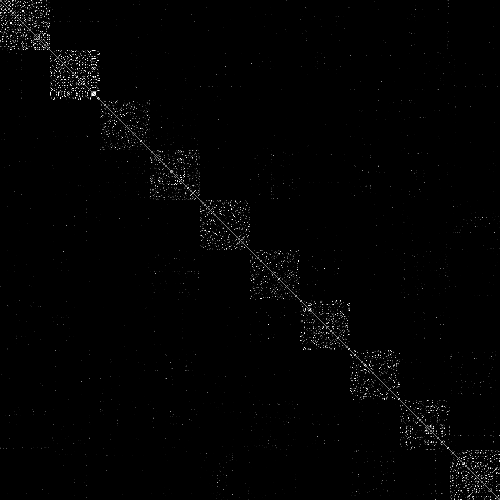}}
\centering{\small (h) SSRSC: \textbf{11.81\%}}
\end{minipage}
}
\vspace{-3mm}
\caption{\linespread{1}\selectfont{Affinity matrices and corresponding average clustering errors by different methods on the handwritten digit images from MNIST~\cite{mnist}.\ 50 images for each digit of $\{0,1,...,9\}$ are used to compute the affinity matrix by different methods.\ All images are normalized to $[0, 1]$ by dividing by the maximal entries in the corresponding affinity matrices.}}
\label{f1}
\end{figure*}

\noindent
\textbf{Comparison Methods}.\ 
We compare the proposed SSRSC with several state-of-the-art SC methods, including SSC~\cite{ssccvpr,sscpami}, LRR~\cite{lrricml,lrrpami}, LRSC~\cite{lrsc}, LSR~\cite{lsr}, SMR~\cite{smr}, S3C~\cite{s3c,s3ctip}, RSIM~\cite{rsim}, SSCOMP~\cite{sscomp}, EnSC~\cite{you2016oracle}, DSC~\cite{ji2017deep}, and ESC~\cite{You2018ECCV}.\ For SSRSC, we fix $s=0.5$, or as reported in ablation study, i.e., $s=0.9$ on Hopkins-155, $s=0.25$ on Extended YaleB, $s=0.4$ on ORL, and $s=0.15$ on MNIST.\ For the SMR method, we use the $J_{1}$ affinity matrix (i.e., (\ref{e7})) in \S\ref{sec:II} as described in~\cite{smr}, for fair comparison. For the other methods, we tune their corresponding parameters on each of the three datasets, i.e., the Hopkins-155 dataset~\cite{benchmark} for motion segmentation, the Extended Yale B dataset~\cite{YaleB} for face clustering, and the MNIST dataset~\cite{mnist} for handwritten digit clustering, to achieve their best clustering results.

\noindent
\textbf{Comparison on Affinity Matrix}.\
The affinity matrix plays a key role in the success of SC methods.\ Here, we visualize the affinity matrix of the proposed SSRSC and the comparison methods, on the SC problem.\ We run the proposed SSRSC algorithm and the competing methods \cite{sscpami,lrrpami,lrsc,lsr,s3c,s3ctip,rsim,sscomp,you2016oracle} on the MNIST dataset~\cite{mnist}.\ The training set contains $6,000$ images for each digit in $\{0,1,...,9\}$. We randomly select $50$ images for each digit, and use the $500$ total images to construct the affinity matrices using these competing SC methods. The results are visualized in Figure~\ref{f1}.\ As can be seen, the affinity matrix of SSRSC shows better connections within each subspace, and generates less noise than most of the other methods.\ Though LSR~\cite{lsr} and RSIM~\cite{rsim} have less noise than our proposed SSRSC, their affinity matrices suffer from strong diagonal entries, indicating that, for these two methods, the data points are mainly reconstructed by themselves.\ With the scaled simplex constraint, the proposed SSRSC algorithm can better exploit the inherent correlations among the data points and achieve better clustering performance than other compared methods.

\noindent
\textbf{Results on Motion Segmentation}.\
We first study how the parameter $s$ influences the average clustering errors of the proposed SSRSC algorithm. The clustering errors with respect to the value of $s$ are plotted in Figure~\ref{f-hopkins}, As can be seen, the proposed SSRSC obtains an average clustering error of 1.53\% when $s=0.5$. Note that SSRSC achieves its lowest average clustering error of 1.04\% when $s=0.9$. The parameter $\lambda$ is set as $\lambda=0.001$.\ We then compare the proposed SSRSC with the other competing SC algorithms~\cite{sscpami,lrrpami,lrsc,lsr,smr,s3c,s3ctip,rsim,sscomp,you2016oracle}.\ The results on the average clustering errors are listed in Table~\ref{t-hopkins}, from which we can see that the proposed SSRSC achieves the lowest clustering error.\ Besides, the speed of the proposed SSRSC approach is only slightly slower than LRSC, LSR, SSCOMP and EnSC, and much faster than the other competing methods.\ Note that the fast speed of SSRSC owes to both efficient solution in each iteration and fewer iterations in the ADMM algorithm.

\begin{figure}[t!]
\centering    
\vspace{-5mm}
\begin{minipage}{0.48\textwidth}
\includegraphics[width=1\textwidth]{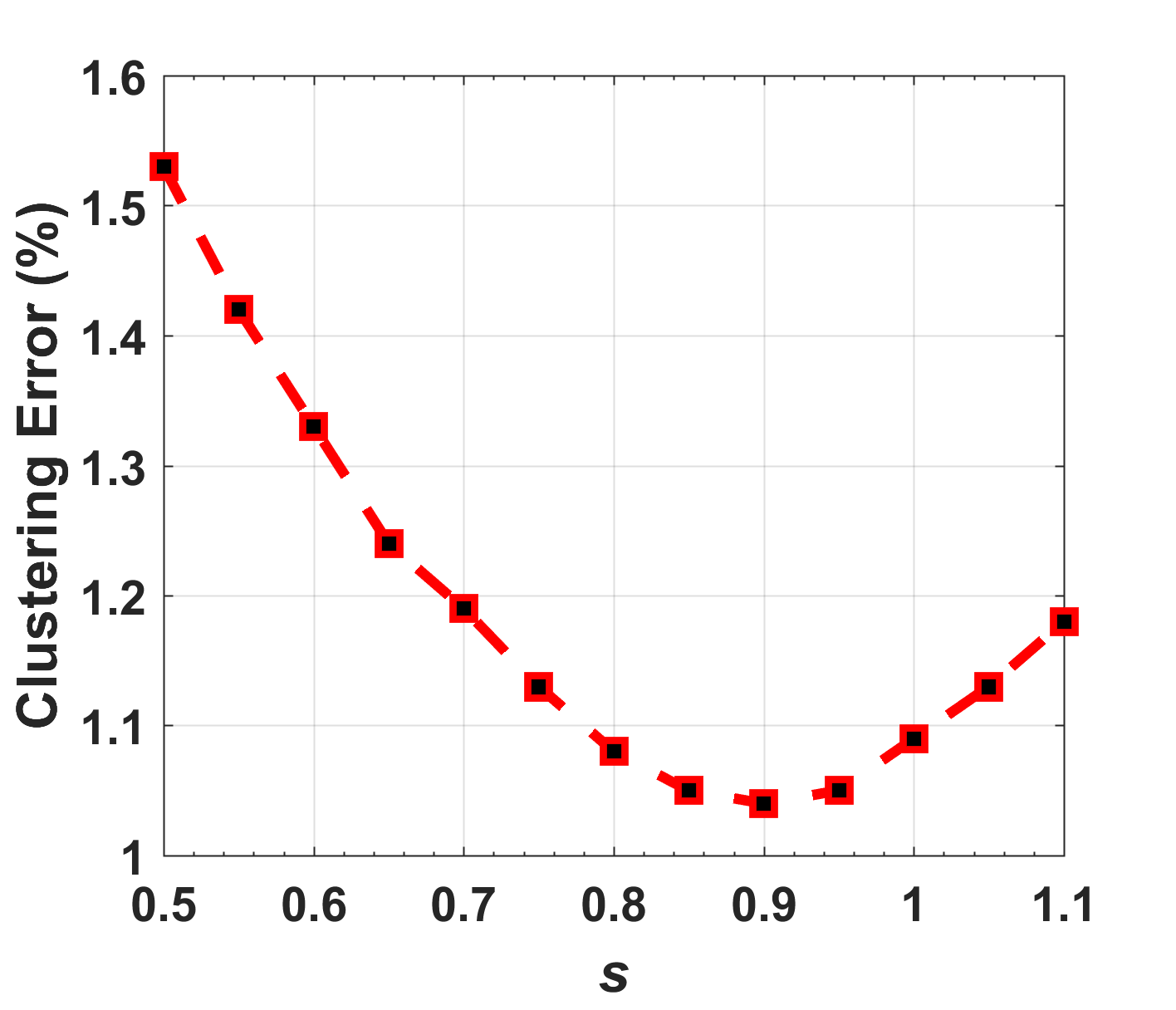}
\end{minipage}
\vspace{-2mm}
\caption{\small Average clustering errors (\%) of the proposed SSRSC algorithm with different scalar $s$, on the Hopkins-155 dataset~\cite{benchmark}.}
\label{f-hopkins}
\end{figure}

\begin{table*}[t!]
\begin{center}
\renewcommand\arraystretch{1}
\footnotesize
\begin{tabular}{c||ccccccccccc}
\Xhline{1pt}
\rowcolor[rgb]{ .85,  .9,  .95}
Method
&
\textbf{SSC}~\cite{sscpami}
&
\textbf{LRR}~\cite{lrrpami}
&
\textbf{LRSC}~\cite{lrsc}
&
\textbf{LSR}~\cite{lsr} 
&
\textbf{SMR}~\cite{smr}
&
\textbf{S3C}~\cite{s3ctip}
&
\textbf{RSIM}~\cite{rsim}
&
\textbf{SSCOMP}~\cite{sscomp}
&
\textbf{EnSC}~\cite{you2016oracle}
&
\textbf{SSRSC}
\\
\hline
Error (\%)
& 2.18 & 3.28 & 5.41 & 2.33 & 2.27 & 2.61 & 1.76 & 5.35 & 1.81 & \textbf{1.53}
\\
Time (s)
& 0.49 & 0.28 & \textbf{0.04} & 0.05 & 0.35 & 2.86 & 0.18 & 0.06 & 0.05 & 0.07
\\
\hline 
\end{tabular}
\vspace{-1mm}
\caption{Average clustering errors (\%) and speed (in seconds) of different algorithms on the Hopkins-155 dataset \cite{benchmark} with the 12-dimensional data points obtained using PCA. SSRSC can achieve clustering error of 1.04\% when $s=0.9$.}
\label{t-hopkins}
\end{center}
\end{table*}

\noindent
\textbf{Results on Human Face Clustering}.\
Here, we compare the proposed SSRSC algorithm with the competing methods, on the commonly used Extended Yale B dataset~\cite{YaleB} and ORL dataset~\cite{orl} for human face clustering.\ On ORL~\cite{orl}, we also compared with the Deep Subspace Clustering (DSC) method, which achieves state-of-the-art performance on this dataset.

We study how the scalar $s$ influences the clustering errors (\%) of the proposed SSRSC algorithm, and take the Extended Yale B dataset~\cite{YaleB} for an example. The average clustering errors with respect to the value of $s$ are plotted in Figure~\ref{f5}.\ As can be seen, SSRSC achieves an average clustering error of 3.26\% when $s=0.5$, and achieves lowest average clustering error of 2.16\% when $s=0.25$. We set $\lambda=0.005$. 

The comparison results of different algorithms are listed in Table~\ref{t-yaleb} and Table~\ref{t-orl}.\ From Table~\ref{t-yaleb}, we observe that, for different numbers ($\{2, 3, 5, 8, 10\}$) of clustering subjects, the average clustering errors of SSRSC are always significantly lower than the other competing methods. For example, when clustering 10 subjects, the average error of SSRSC (when fixing $s=0.5$) is $5.10\%$, while the errors of the other methods are 10.94\% for SSC, 28.54\% for LRR, 30.05\% for LRSC, 28.59\% for LSR, 28.18\% for SMR, 5.16\% for S3C, 6.56\% for RSIM, 14.80\% for SSCOMP, and 5.96\% for EnSC, respectively. The performance gain of SSRSC (5.10\%) over its baseline LSR (28.59\%) demonstrates the power of simplex representation on human face clustering. In terms of running time, the proposed SSRSC algorithm is shown to be comparable to SSCOMP, which is the most efficient algorithm among all competing methods.\ In Table~\ref{t-orl}, we can observe that, the proposed SSRSC method achieves lower clustering error than previous methods except the deep learning based DSC.\ This demonstrates the advantages of the proposed SSR representation over previous sparse or low rank representation frameworks on human face clustering.

\begin{figure}[t!]
    \centering
    \vspace{-5mm}
\begin{minipage}{0.48\textwidth}
\includegraphics[width=1\textwidth]{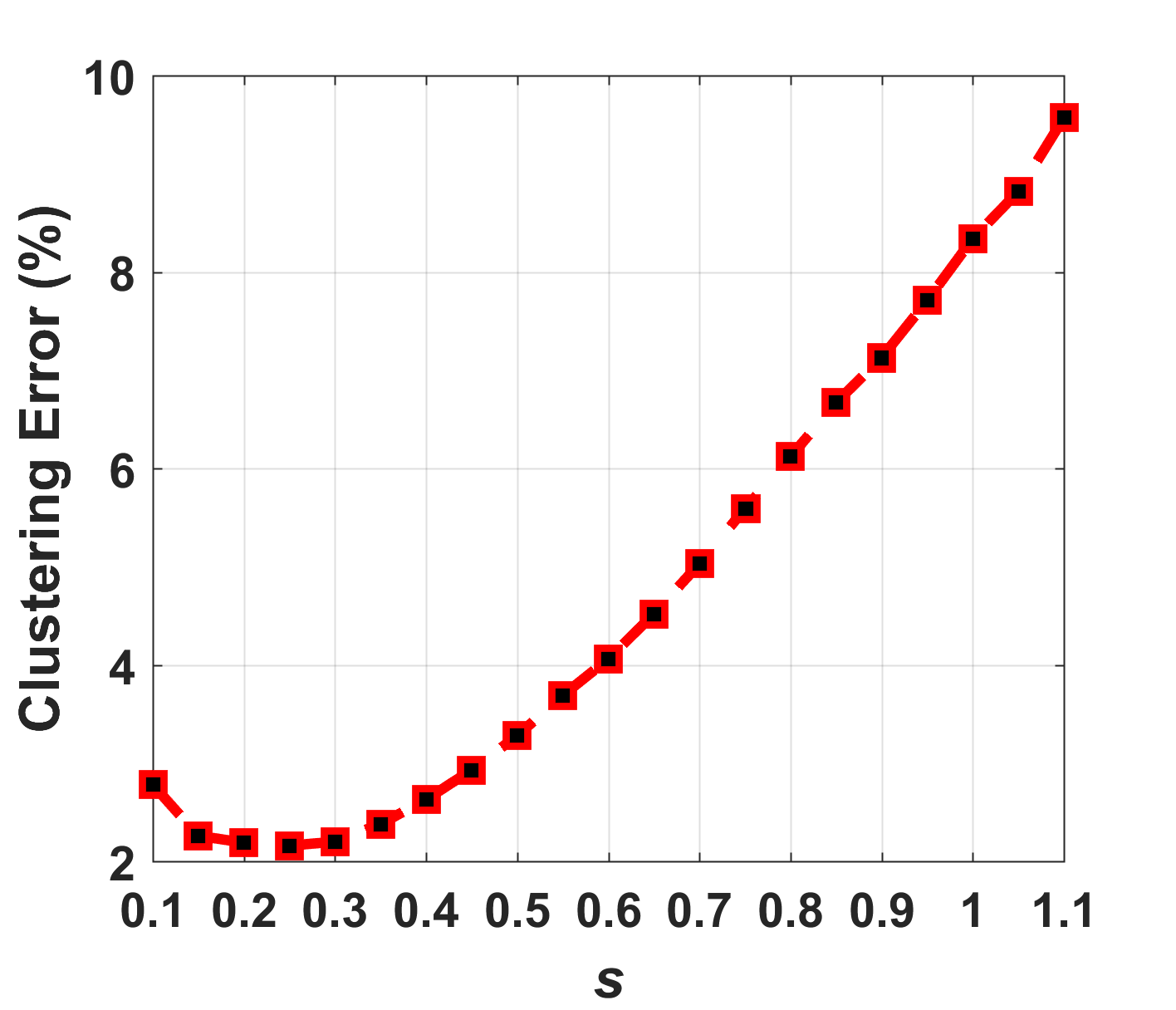}
\end{minipage}
    \vspace{-2mm}
   \caption{\small Average clustering errors (\%) of the proposed SSRSC algorithm with different scalar $s$, on the Extended Yale B dataset~\cite{YaleB}.}
\label{f5}
\end{figure}

\begin{table*}[t!]
\begin{center}
\footnotesize 
\begin{tabular}{r||cc|cc|cc|cc|cc}
\Xhline{1pt}  
\rowcolor[rgb]{ .85,  .9,  .95}
& 
\multicolumn{2}{c|}{{2 Subjects}} 
& 
\multicolumn{2}{c|}{{3 Subjects}} 
& 
\multicolumn{2}{c|}{{5 Subjects}} 
& 
\multicolumn{2}{c|}{{8 Subjects}} 
& 
\multicolumn{2}{c}{{10 Subjects}} 
\\
\rowcolor[rgb]{ .85,  .9,  .95}
\multirow{-2}{*}{\# of Subjects}
&  
\multicolumn{1}{c}{\centering{Error (\%)}}
&  
\multicolumn{1}{c|}{\centering{Time (s)}}
&  
\multicolumn{1}{c}{\centering{Error (\%)}}
&  
\multicolumn{1}{c|}{\centering{Time (s)}}
&  
\multicolumn{1}{c}{\centering{Error (\%)}}
&  
\multicolumn{1}{c|}{\centering{Time (s)}}
&  
\multicolumn{1}{c}{\centering{Error (\%)}}
&  
\multicolumn{1}{c|}{\centering{Time (s)}}
&  
\multicolumn{1}{c}{\centering{Error (\%)}}
&  
\multicolumn{1}{c}{\centering{Time (s)}}
\\
\hline
\textbf{SSC}~\cite{sscpami}
& 1.86& 0.17& 3.10& 0.31& 4.31& 0.74 & 5.85& 2.72& 10.94& 4.27
\\
\textbf{LRR}~\cite{lrrpami}
& 4.99& 0.15& 6.86& 0.29& 14.04 & 0.65 & 24.18& 1.30 & 28.54 & 1.81 
\\
\textbf{LRSC}~\cite{lrsc}
& 4.42 &\textbf{0.03}& 6.14 & 0.06 & 13.06 & 0.13& 24.12& 0.28& 30.05 & 0.45
\\
\textbf{LSR}~\cite{lsr}
& 4.98&\textbf{0.03}& 6.87 & 0.05 & 14.14& 0.12& 24.39& 0.25& 28.59 & 0.36 
\\
\textbf{SMR}~\cite{smr}
& 4.68& 0.21& 6.56& 0.30& 13.78 & 0.76& 24.25& 2.00& 28.18& 2.59
\\
\textbf{S3C}~\cite{s3ctip}
& 1.27& 1.31& 2.71& 2.58& 3.41& 3.52 & 4.13& 7.14& 5.16& 13.98
\\
\textbf{RSIM}~\cite{rsim}
& 2.17& 0.09& 2.96 & 0.17& 4.13& 0.47& 5.82& 1.77& 6.56& 3.19
\\
\textbf{SSCOMP}~\cite{sscomp}
& 2.76& \textbf{0.03}& 5.35& \textbf{0.04} & 7.89& \textbf{0.08}& 11.40& 0.18 & 14.80& \textbf{0.27}
\\
\textbf{EnSC}~\cite{you2016oracle}
& 1.64& \textbf{0.03}& 2.91& 0.05& 3.84& 0.11& 5.84& 0.20& 5.96&0.29
\\
\hline
\textbf{SSRSC}
& 1.26& \textbf{0.03}& 1.58& \textbf{0.04}& 2.70& \textbf{0.08} & 5.66& \textbf{0.17} & 5.10& \textbf{0.27}
\\
\textbf{SSRSC ($s=0.25$)}
& \textbf{0.96}& \textbf{0.03} & \textbf{1.32}& \textbf{0.04}& \textbf{2.22} & \textbf{0.08} & \textbf{3.00}& \textbf{0.17} & \textbf{3.18} & \textbf{0.27}
\\
\hline 
\end{tabular} 
\vspace{-1mm}
\caption{Average clustering errors (\%) and speed (in seconds) of different algorithms on the Extended Yale B dataset~\cite{YaleB} with the $6n$-dimensional ($n$ is the number of subjects) data points obtained using PCA.}
\label{t-yaleb}
\vspace{-3mm}
\end{center}
\end{table*}

\begin{table*}[t!]
\begin{center}
\renewcommand\arraystretch{1}
\footnotesize
\begin{tabular}{c||ccccccccccc}
\Xhline{1pt}
\rowcolor[rgb]{ .85,  .9,  .95}
Method
&
\textbf{SSC}~\cite{sscpami}
&
\textbf{LRR}~\cite{lrrpami}
&
\textbf{LRSC}~\cite{lrsc}
&
\textbf{LSR}~\cite{lsr} 
&
\textbf{SMR}~\cite{smr}
&
\textbf{RSIM}~\cite{rsim}
&
\textbf{SSCOMP}~\cite{sscomp}
&
\textbf{EnSC}~\cite{you2016oracle}
&
\textbf{DSC}~\cite{ji2017deep}
&
\textbf{SSRSC}
\\
\hline
Error (\%)
& 32.50 & 38.25 & 32.50 & 27.25 & 25.75 & 26.25 & 36.00 & 23.25 & \textbf{14.00} & 21.75
\\
Time (s) & 12.32 & 6.01 & 1.43 & 1.21 & 7.97 & 9.81 & 0.77 & 0.82 & \textbf{0.23} & 0.89
\\
\hline
\end{tabular}
\vspace{-1mm}
\caption{Average clustering errors (\%) and speed (in seconds) of different algorithms on the ORL dataset~\cite{orl} with the $32\times32$ data points obtained by resizing. SSRSC can achieve clustering error of 21.25\% when $s=0.4$.}
\label{t-orl}
\vspace{-3mm}
\end{center}
\end{table*}

\begin{table*}[ht!]
\begin{center}
\footnotesize 
\begin{tabular}{r||rr|rr|rr|rr|rr}
\Xhline{1pt}   
\rowcolor[rgb]{ .85,  .9,  .95}
& 
\multicolumn{2}{c|}{{500}} 
& 
\multicolumn{2}{c|}{{1,000}} 
& 
\multicolumn{2}{c|}{{2,000}} 
& 
\multicolumn{2}{c|}{{4,000}} 
& 
\multicolumn{2}{c}{{6,000}} 
\\           
\rowcolor[rgb]{ .85,  .9,  .95}
\multirow{-2}{*}{\# of Images}
&  
\multicolumn{1}{c}{\centering{Error (\%)}}
&  
\multicolumn{1}{c|}{\centering{Time (s)}}
&  
\multicolumn{1}{c}{\centering{Error (\%)}}
&  
\multicolumn{1}{c|}{\centering{Time (s)}}
&  
\multicolumn{1}{c}{\centering{Error (\%)}}
&  
\multicolumn{1}{c|}{\centering{Time (s)}}
&  
\multicolumn{1}{c}{\centering{Error (\%)}}
&  
\multicolumn{1}{c|}{\centering{Time (s)}}
&
\multicolumn{1}{c}{\centering{Error (\%)}}
&
\multicolumn{1}{c}{\centering{Time (s)}}
\\
\hline
\textbf{SSC}~\cite{sscpami}
& 16.99 & 27.36 & 15.95 & 49.25 & 14.42& 94.84  & 14.00& 239.95 & 14.40 & 423.47 
\\
\textbf{LRR}~\cite{lrrpami}
& 17.97& 15.64 & 16.88& 28.14 & 16.63 & 54.19& 15.42 & 137.11& 15.45 & 241.98
\\
\textbf{LRSC}~\cite{lrsc}
& 24.16& 0.34  & 21.58 & 0.96 & 21.91 & 4.74 &20.94 & 27.40 & 20.09 & 78.21 
\\
\textbf{LSR}~\cite{lsr}
& 24.98& \textbf{0.29} & 20.24 & 0.91& 20.56 & 4.74& 22.24& 28.86 & 20.12 & 83.14 
\\
\textbf{SMR}~\cite{smr}
& 18.45& 19.56 & 13.97& 35.18 & 9.46& 67.75& 9.14& 171.39&7.36 &302.48
\\
\textbf{S3C}~\cite{s3ctip}
& 15.92 & 135.20 & 12.23 & 263.33& 10.54 & 521.93 & 9.46& 1308.50 & 9.34& 2544.41
\\
\textbf{RSIM}~\cite{rsim}
& 18.13& 10.06& 15.70& 18.09& 11.48& 34.84 & 10.53 & 88.14& 10.21 & 155.56
\\
\textbf{SSCOMP}~\cite{sscomp}
& 16.36& 0.32 & 13.33 & \textbf{0.88}& 9.40& \textbf{4.50} &8.78& \textbf{26.94}& 8.75& \textbf{76.84} 
\\
\textbf{EnSC}~\cite{you2016oracle}
& 13.45 & 0.37& 9.31& 0.98& 7.69& 4.67& 6.71& 27.14& 5.86& 77.24
\\
\hline
\textbf{SSRSC}
& 11.81& 0.34 & 6.90& {0.98} & 5.65& 5.01 & 5.31& 30.06& 4.53& {85.68}
\\
\textbf{SSRSC ($s=0.15$)}
& \textbf{10.29} & {0.35} & \textbf{5.40}& {0.99}& \textbf{4.36}& 5.03 & \textbf{3.09}& 30.11& \textbf{2.52}& {85.74}  
\\
\hline
\end{tabular} 
\vspace{-1mm}
\caption{Average clustering errors (\%) and speed (in seconds) of different algorithms on the MNIST dataset~\cite{mnist}.\ The features are extracted from a scattering network and projected onto a 500-dimensional subspace using PCA.\ The experiments are repeated for 20 times and the average results are reported.}
\label{t-mnist}
\vspace{-3mm}
\end{center}
\end{table*}

\begin{table*}[t!]
\begin{center}
\renewcommand\arraystretch{1}
\footnotesize
\begin{tabular}{r||ccccccccccc}
\Xhline{1pt}
\rowcolor[rgb]{ .85,  .9,  .95}
Method
&
\textbf{SSC}~\cite{sscpami}
&
\textbf{LRR}~\cite{lrrpami}
&
\textbf{LRSC}~\cite{lrsc}
&
\textbf{LSR}~\cite{lsr} 
&
\textbf{SMR}~\cite{smr}
&
\textbf{RSIM}~\cite{rsim}
&
\textbf{SSCOMP}~\cite{sscomp}
&
\textbf{EnSC}~\cite{you2016oracle}
&
\textbf{ESC}~\cite{You2018ECCV}
&
\textbf{SSRSC}
\\
\hline
Error (\%)
&30.11 & 33.89 & 31.51 & 32.21 & 29.58 & 30.15 & 35.66 & 26.23 &25.42& \textbf{23.45}
\\
Time (s) & 863.29 & 495.13 &162.32 &172.28 & 621.89 & 320.14 & 159.22 & 160.98 & 
\textbf{157.67} & 
{186.23}
\\
\hline
\end{tabular}
\vspace{-1mm}
\caption{Average clustering errors (\%) and speed (in seconds) of different algorithms on the EMNIST dataset \cite{emnist}.\ The features are extracted from a scattering network and projected onto a 500-dimensional subspace using PCA.\ The experiments are repeated for 10 times and the average results are reported.\ SSRSC can achieve clustering error of 22.76\% when $s=0.35$.}
\vspace{-5mm}
\label{t-emnist}
\end{center}
\end{table*}

\noindent
\textbf{Results on Handwritten Digit Clustering}.\
For the digit clustering problem, we evaluated the proposed SSRSC with the competing methods on the MNIST dataset~\cite{mnist} and the EMNIST dataset~\cite{emnist}.\ We follow the experimental settings in SSCOMP~\cite{sscomp}.\ On the MNIST~\cite{mnist}, the testing data consists of a randomly chosen number of
$N_{i}\in\{50, 100, 200, 400, 600\}$ images for each of the 10 digits, while on the EMNIST~\cite{emnist}, the testing data consists of $N_{i}=500$ images for each of the 26 digits or letters.\ The feature vectors are extracted from the scattering convolution network \cite{scn}.\ The features extracted from the digit/letter images are originally $3,472$-dimensional, and are projected onto a $500$-dimensional subspace using PCA.

\begin{figure}[t]
\centering
    \vspace{-5mm}
\begin{minipage}{0.48\textwidth}
\includegraphics[width=1\textwidth]{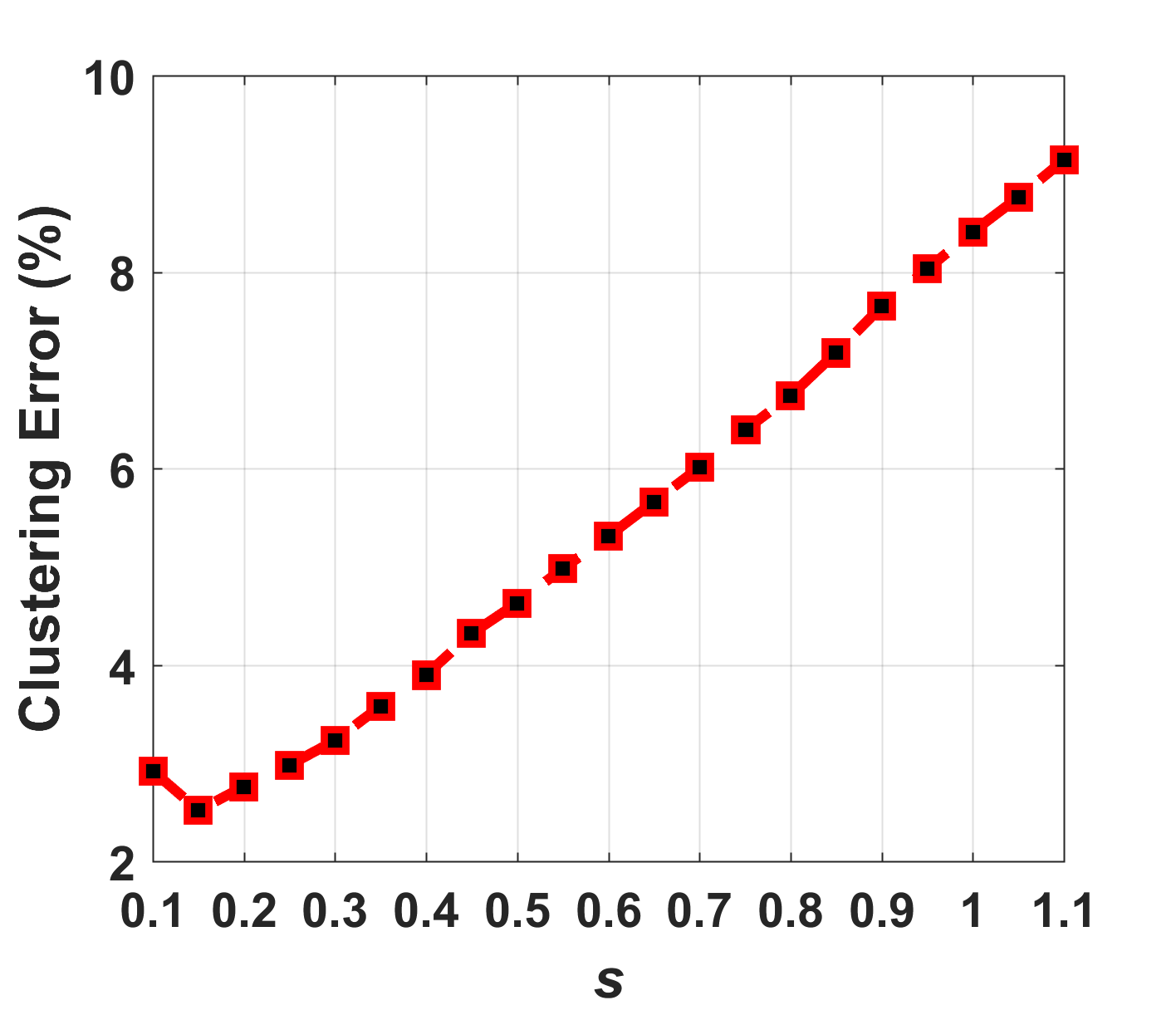}
\end{minipage}
    \vspace{-2mm}
\caption{\small Average clustering errors (\%) of the proposed SSRSC with different $s$ on MNIST~\cite{mnist}, with $600$ images for each digit.}
\label{f7}
\end{figure}

We evaluate the influence of the scalar $s$ on the average clustering error (\%) of SSRSC on the MNIST dataset. The curve of clustering errors w.r.t. $s$ is plotted in Figure~\ref{f7}. Average clustering errors are computed over 20 trials sampled from $6,000$ randomly selected images ($N_{i}=600$).\ As can be seen, the proposed SSRSC algorithm achieves an average clustering error of 4.53\%, and obtains its lowest average clustering error (2.52\%) when $s=0.15$, with 600 images for each digit from the MNIST dataset. Similar results can be observed on the EMNIST dataset~\cite{emnist}. The parameter $\lambda$ is set as $\lambda=0.01$. 

We list the average clustering errors (\%) of the comparison methods in Tables~\ref{t-mnist} and~\ref{t-emnist}. It can be seen that the proposed SSRSC algorithm performs better than all other competing methods. For example, on MNIST dataset, SSRSC achieves clustering errors of 11.81\%, 6.90\%, 5.65\%, 5.31\%, and 4.53\% when the number of digit images are 500, 1000, 2000, 4000, and 6000, respectively. On EMNIST dataset, SSRSC achieves clustering error of 23.45\%, when the number of images for each digit/letter are $500$. The results of the proposed SSRSC on the two datasets are much better than the other competing algorithms, including the recently proposed EnSC and ESC, respectively. This validates the advantages of the proposed scaled simplex representation over previous sparse or low-rank representation, for hand-written digit/letter clustering.

\begin{table*}[t]
\begin{center}
\small
\begin{tabular}{r||c|c|c|c|c}
\Xhline{1pt}
\rowcolor[rgb]{ .85,  .9,  .95}
&&& \multicolumn{2}{c|}{{Constraint}} &
\\
\rowcolor[rgb]{ .85,  .9,  .95}
\multicolumn{1}{c||}{\multirow{-2}{*}{Model}}
&
\multirow{-2}{*}{Data Term}
&
\multirow{-2}{*}{Regularization Term}
&
$\bm{C}\ge0$ 
&
$\bm{1}^{\top}\bm{C}=s\bm{1}^{\top}$
& 
\multirow{-2}{*}{Solution}
\\
\hline
LSR~(\ref{e-lsr})
& 
\multirow{4}{*}{$\|\bm{X}-\bm{X}\bm{C}\|_{F}^{2}$ }
& 
\multirow{4}{*}{$\|\bm{C}\|_{F}^{2}$} 
& 
\xmark
& 
\xmark
& 
$\overline{\bm{C}}=(\bm{X}^{\top}\bm{X}+\lambda\bm{I})^{-1}\bm{X}^{\top}\bm{X}$ 
\\
NLSR~(\ref{e-nlsr})
& 
& 
& 
\cmark
& 
\xmark
& 
See Supplementary File
\\
SLSR~(\ref{e-slsr})
& 
& 
& 
\xmark
& 
\cmark
& 
See Supplementary File
\\
SSRSC\ ~(\ref{e9})
& 
& 
&
\cmark
&
\cmark
& 
See \S\ref{sec:III-B} 
\\
\hline 
\end{tabular}
\vspace{-1mm}
\caption{Summary of the proposed SSRSC and three baseline methods LSR, NLSR, and SLSR.}
\vspace{-3mm}
\label{t-summary}
\end{center}
\end{table*}

\begin{table}[t]
\begin{center}
\renewcommand\arraystretch{1}
\small
\begin{tabular}{r||ccccc}
\Xhline{1pt}
\rowcolor[rgb]{ .85,  .9,  .95}
Method
&
\textbf{LSR}
&
\textbf{NLSR}
&
\textbf{SLSR}
&
\textbf{SSRSC} 
\\
\hline
Error (\%)
& 3.67 & 1.75 & 3.16 & \textbf{1.04}
\\
\hline
\end{tabular}
\vspace{-1mm}
\caption{Average clustering errors (\%) of LSR, NLSR, SLSR ($s$$=$$0.9$), and SSRSC ($s$$=$$0.9$) on Hopkins-155~\cite{benchmark}, with the 12-dimensional data points obtained by PCA.}
\label{t-ablahopkins}
\end{center}
\end{table}  

\begin{table}[t]
\vspace{-0mm}
\begin{center}
\renewcommand\arraystretch{1}
\small
\begin{tabular}{r||rrrr}
\Xhline{1pt}
\rowcolor[rgb]{ .85,  .9,  .95}
$n$
&
\textbf{LSR}
&
\textbf{NLSR}
&
\textbf{SLSR}
&
\textbf{SSRSC}
\\
\hline
2 
& 4.98 & 3.46 & 4.35 & \textbf{0.96}
\\
\hline
3
& 6.87 & 4.95 & 6.56 & \textbf{1.32}
\\
\hline
5
& 14.14 & 10.03 & 13.77 & \textbf{2.22}
\\
\hline
8
& 24.39 & 17.74 & 22.14 & \textbf{3.00}
\\
\hline
10
& 28.59 & 20.29 & 26.31 & \textbf{3.18}
\\
\hline
\end{tabular}
\vspace{-1mm}
\caption{Average clustering errors (\%) of LSR, NLSR, SLSR ($s$$=$$0.22$), and SSRSC ($s$$=$$0.25$) on Extended Yale B~\cite{YaleB}, with the 6$n$-dimensional ($n$ is the number of subjects) data points obtained by PCA.}
\label{t-ablayaleb}
\end{center}
\end{table}

\subsection{Ablation Study: Effectiveness of Simplex Representation}
\label{sec:IV-D}
We perform comprehensive ablation studies on the Hopkins-155~\cite{benchmark}, Extended Yale B~\cite{YaleB}, and MNIST~\cite{mnist} datasets to validate the effectiveness of the proposed SSRSC model.\ SSRSC includes two constraints, i.e., $\bm{c}\ge0$ and $\bm{1}^{\top}\bm{c}=s$.\ To analyze the effectiveness of each constraint, we compare the proposed SSRSC model with several baseline methods. 

The first baseline is the trivial Least Square Regression (LSR) model as follows:
\begin{equation}
\begin{split}
\label{e-lsr}
\textbf{LSR:}
\ 
\min_{\bm{C}}
\|
\bm{X}
-
\bm{X}\bm{C}
\|_{F}^{2}
+
\lambda
\|
\bm{C}
\|
_{F}^{2}.
\end{split}
\end{equation}
By removing each individual constraint in the simplex constraint of the proposed SSRSC model~(\ref{e9}), we have the second baseline method, called the Non-negative Least Square Regression (NLSR) model:
\begin{equation}
\begin{split}
\label{e-nlsr}
\textbf{NLSR:}
\ 
\min_{\bm{C}}
\|
\bm{X}
-
\bm{X}\bm{C}
\|_{F}^{2}
+
\lambda
\|
\bm{C}
\|
_{F}^{2}
\quad
\text{s.t.}
\quad
\bm{C}
\ge
0
.
\end{split}
\end{equation}

The NLSR model can be formulated by either removing the scaled affinity constraint $\bm{1}^{\top}\bm{c}=s$ from SSRSC or adding a non-negative constraint to the LSR model~(\ref{e-lsr}). For structural clearance of this manuscript, we put the solution of the NLSR model~(\ref{e-nlsr}) in the \textsl{Supplementary File}. We can also remove the non-negative constraint $\bm{C}\ge0$ in the scaled simplex set $\{\bm{C}\in\mathbb{R}^{N\times N}|\bm{C}\ge0,\bm{1}^{\top}\bm{C}=s\bm{1}^{\top}\}$, and obtain a Scaled-affine Least Square Regression (SLSR) model:
\begin{equation}
\begin{split}
\label{e-slsr}
\textbf{SLSR:}
\ 
\min_{\bm{C}}
\|
\bm{X}
-
\bm{X}\bm{C}
\|_{F}^{2}
+
\lambda
\|
\bm{C}
\|
_{F}^{2}
\quad
\text{s.t.}
\quad
\bm{1}^{\top}\bm{C}=s\bm{1}^{\top}
.
\end{split}
\end{equation}
The solution of the SLSR model~(\ref{e-slsr}) is also provided in the \textsl{Supplementary File}. 

The proposed SSRSC method, as well as the three baseline methods, can be expressed in a standard form:
\begin{equation}
\label{e-form}
\min_{\bm{C}} \textsl{Data Term} + \textsl{Regularization Term} 
\ 
\text{s.t.}
\ 
\textsl{Constraints}
.
\end{equation}
For the NLSR, SLSR, and SSRSC methods, the data and regularization terms are the same as that of the basic LSR model~(\ref{e9}), i.e., $\|\bm{X}-\bm{X}\bm{C}\|_{F}^{2}$ and $\lambda\|\bm{C}\|_{F}^{2}$, respectively. The difference among these comparison methods lies in the~\textsl{Constraints}.\ In Table~\ref{t-summary}, we summarize the proposed SSRSC, and the three baseline methods LSR, NLSR, and SLSR. 

In Tables~\ref{t-ablahopkins}-\ref{t-ablamnist}, we list the average clustering errors (\%) of the proposed SSRSC and three baseline methods on the Hopkins-155~\cite{benchmark}, Extended Yale B~\cite{YaleB}, and MNIST~\cite{mnist} datasets.\ Note that here we tune the parameter $s$ to achieve the best performance of SSRSC on different datasets.

\noindent
\textbf{Effectiveness of Non-negative Constraint}.\
The effectiveness of non-negative constraint $\bm{C}\ge0$ in the proposed SSRSC model can be validated from two complementary aspects.\ First, it can be validated by evaluating the performance improvement of the baseline methods NLSR~(\ref{e-nlsr}) over LSR~\cite{lsr}.\ Since the only difference between them is the additional non-negative constraint in NLSR~(\ref{e-nlsr}), the performance gain of NSLR over LSR can directly reflect the effectiveness of the non-negative constraint.\ Second, the effectiveness of non-negative constraint can also be validated by comparing the performance of the proposed SSRSC~(\ref{e9}) and the baseline method SLSR~(\ref{e-slsr}).\ Since SLSR~(\ref{e-slsr}) lacks a non-negative constraint when compared to SSRSC, the performance gain of SSRSC over SLSR should be due to the introduced non-negativity.

The results listed in Table~\ref{t-ablahopkins}, show that, on the Hopkins-155 dataset, the baseline methods LSR and NLSR achieve average clustering errors of 3.67\% and 1.75\%, respectively. This demonstrates that, by adding the non-negative constraint to LSR, the resulting baseline method NLSR has a significantly reduced clustering error. Besides, the average clustering errors of SSRSC ($s=0.9$) and SLSR ($s=0.9$) are 1.04\% and 3.16\%, respectively. This shows that, if we remove the non-negative constraint from SSRSC, the resulting SLSR will have a significantly lower clustering performance. Similar trends can be found from Tables~\ref{t-ablayaleb} and~\ref{t-ablamnist} on the Extended Yale B dataset~\cite{YaleB} and the MNIST dataset~\cite{mnist}, respectively. These results validate the contribution of the non-negative constraint to the success of the proposed SSRSC model.

\begin{table}[t]
\vspace{-0mm}
\begin{center}
\renewcommand\arraystretch{1}
\small
\begin{tabular}{r||rrrr}
\Xhline{1pt}
\rowcolor[rgb]{ .85,  .9,  .95}
\# of Images
&
\textbf{LSR}
&
\textbf{NLSR}
&
\textbf{SLSR}
&
\textbf{SSRSC}
\\
\hline
500 
& 24.98 & 15.29 & 22.00 & \textbf{10.29}
\\
\hline
1,000
& 20.24 & 12.66 & 18.59 & \textbf{5.40}
\\
\hline
2,000
& 20.56 & 10.13 & 16.87 & \textbf{4.36}
\\
\hline
4,000
& 22.24 & 9.07 & 15.48 & \textbf{3.09}
\\
\hline
6,000
& 20.12 & 8.34 & 14.23 & \textbf{2.52}
\\
\hline 
\end{tabular}
\vspace{-1mm}
\caption{Average clustering errors (\%) of LSR, NLSR, SLSR ($s=0.24$), and SSRSC ($s=0.15$) on the MNIST dataset~\cite{mnist}.\ The features of data point are extracted from a scattering network~\cite{scn} and projected onto a 500-dimensional subspace obtained using PCA.\ The experiments are independently repeated 20 times.}
\vspace{-3mm}
\label{t-ablamnist}
\end{center}
\end{table}

\noindent
\textbf{Effectiveness of Scaled-Affine Constraint}.\
The effectiveness of scaled-affine constraint can be validated from two aspects. First, it can be validated by comparing the baseline methods LSR~(\ref{e-lsr}) and SLSR~(\ref{e-slsr}), which are summarized in Table~\ref{t-summary}.\ Since the only difference between them is that SLSR contains an additional scaled affine constraint over LSR, the performance gain of SLSR over LSR can directly validate the scaled affine constraint's effectiveness.\ Second, the effectiveness can also be validated by comparing the proposed SSRSC~(\ref{e9}) and the baseline NLSR~(\ref{e-nlsr}).\ This is because NLSR removes the scaled affine constraint from SSRSC. 

From the results listed in Tables~\ref{t-ablahopkins}, we observe that, on the Hopkins-155 dataset~\cite{benchmark}, the proposed SSRSC ($s=0.9$) can achieve lower clustering error (1.04\%) than the baseline method NLSR (1.75\%). Similar conclusions can be drawn from Tables~\ref{t-ablayaleb} and~\ref{t-ablamnist} for experiments on the Extended Yale B dataset~\cite{YaleB} and the MNIST dataset~\cite{mnist}, respectively. All these comparisons demonstrate that the scaled affine constraint is another essential factor for the success of SSRSC.

\noindent
\textbf{Effectiveness of the Simplex Constraint}.\
We find that the proposed simplex constraint, i.e., the integration of the non-negative and scaled affine constraints, can further boost the performance of SC. This can be validated by comparing the performance of the proposed SSRSC~(\ref{e9}) and the baseline method LSR~(\ref{e-lsr}).\ The results listed in Tables~\ref{t-ablahopkins}-\ref{t-ablamnist} show that, on all three commonly used datasets, the proposed SSRSC can achieve much lower clustering errors than the baseline method LSR. For example, on the MNIST dataset, the clustering errors of SSRSC ($s=0.15$) are $10.29\%$, $5.40\%$, $4.36\%$, $3.09\%$, and $2.52\%$ when we randomly select $50, 100, 200, 400$, and $600$ images for each of the $10$ digits, respectively. Meanwhile, the corresponding clustering errors of the baseline method LSR are $24.98\%$, $20.24\%$, $20.56\%$, $22.24\%$, $20.12\%$, respectively. SSRSC performs much better than LSR in all cases. Similar trends can also be found in Table~\ref{t-ablahopkins} for the Hopkins-155 dataset and in Table~\ref{t-ablayaleb} for the Extended Yale B dataset.

\noindent
\textbf{Influence of diagonal constraint}.\ One key difference of our SSRSC model to previous models is that, in SSRSC, we get rid of diagonal constraint $\text{diag}(\bm{C})=\bm{0}$ considering that the practical data is often noisy.\ To achieve this, we compare with method on the Hopkins155 dataset, together with the SSRSC model with additional diagonal constraint of $\text{diag}(\bm{C})=\bm{0}$ (we call this method SSRSC-diag).\ The results are listed in Table~\ref{t-diag}.\ One can see that the variant SSRSC-diag achieves inferior performance with the original SSRSC.\ We also compare with a state-of-the-art clustering method, i.e., RGC~\cite{rgc2019}, designed specifically for clustering noisy data.\ RGC achieves slightly better performance than SSRSC-diag, but is still inferior to our SSRSC without diagonal constraint.\ This demonstrates the necessity of getting rid of the diagonal constraint in our SSRSC model to deal with noisy data.

\begin{table}[htp]
\begin{center}
\renewcommand\arraystretch{1}
\small
\begin{tabular}{cccccc}
\Xhline{1pt}
\rowcolor[rgb]{ .85,  .9,  .95}
Method
&
\textbf{SSRSC}
&
\textbf{SSRSC-diag}
&
\textbf{RGC}~\cite{rgc2019}
\\
\hline
Error (\%)
& \textbf{1.04} & 1.87 & 1.76
\\
\Xhline{1pt}
\end{tabular}
\end{center}
\vspace{-3mm}
\caption{Average clustering errors (\%) of SSRSC ($s=0.9$), SSRSC-diag ($s=0.8$) and RGC on Hopkins-155 dataset~\cite{benchmark}, with the 12-dimensional data points obtained using PCA.}
\label{t-diag}
\end{table}

\section{Conclusion}
\label{sec:V}
In this paper, we proposed a scaled simplex representation (SSR) based model for spectral clustering based subspace clustering (SC).\ Specifically, we introduced the non-negative and scaled affine constraints into a simple least square regression model.\ The proposed SSRSC model can reveal the inherent correlations among data points in a highly discriminative manner.\ Based on the SSRSC model, a novel spectral clustering based algorithm was developed for subspace clustering.\ Extensive experiments on three benchmark clustering datasets demonstrated that the proposed SSRSC algorithm is very efficient and achieves better clustering performance than state-of-the-art subspace clustering algorithms.\ The significant improvements of SSRSC over the baseline models demonstrated the effectiveness of the proposed simplex representation.\

Our work can be extended in at least three directions.\ First, it is promising to extend the proposed SSRSC model to nonlinear version by using kernel approaches~\cite{peng2017integrating,kangkernel,kanglowrankkernel,peng2019discriminative}.\ Second, it is worthy to accelerate the proposed algorithm for scalable subspace clustering~\cite{You2018ECCV,Peng_2019_CVPR}.\ Third, adapting the proposed SSRSC model for imbalanced datasets~\cite{You2018ECCV,peng2019discriminative} is also a valuable direction.


{
\small
\bibliographystyle{unsrt}
\bibliography{SSRSC}
}

\vspace{-10mm}
\begin{IEEEbiography}[{\includegraphics[width=1in,height=1.25in,clip,keepaspectratio]{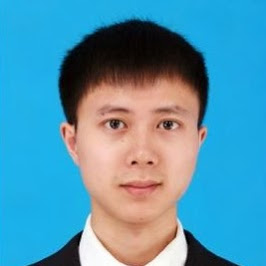}}]{Jun Xu} received his B.Sc. and M.Sc. degrees in Mathematics from the School of Mathematics Science, Nankai University, Tianjin, China, in 2011 and 2014, respectively. He received his Ph.D degree in Department of Computing, The Hong Kong Polytechnic University, Hong Kong SAR, China, in 2018. He was a Research Scientist in Inception Institute of Artificial Intelligence, Abu Dhabi, UAE, from 2018 to 2019. He is currently an Assistant Professor in College of Computer Science, Nankai University, Tianjin, China. 
\end{IEEEbiography}

\vspace{-10mm}
\begin{IEEEbiography}[{\includegraphics[width=1in,height=1.25in,clip,keepaspectratio]{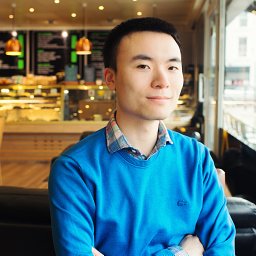}}]{Mengyang Yu} received the B.Sc. and M.Sc. degrees from the School of Mathematical Sciences, Peking University, Beijing, China, in 2010 and 2013, respectively. He received his Ph.D. degree from the Department of Computer Science and Digital Technologies, Northumbria University, Newcastle upon Tyne, U.K., in 2017. He was a Postdoc in ETH Zurich from 2017-2018, working with Prof. Luc Van Gool. He is now a Research Scientist in Inception Institute of Artificial Intelligence, Abu Dhabi, UAE. His research interests include computer vision, machine learning, and data mining.
\end{IEEEbiography}

\vspace{-10mm}
\begin{IEEEbiography}[{\includegraphics[width=1in,height=1.25in,clip,keepaspectratio]{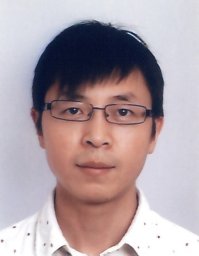}}]{Ling Shao} (M'09, SM'10) was a Chair Professor with School of Computing Sciences, University of East Anglia, Norwich, U.K., from 2016 to 2018, a Professor with Northumbria University from 2014 to 2016, a Senior Lecturer with the Department of Electronic and Electrical Engineering, The University of Sheffield, from 2009 to 2014, and a Senior Scientist with Philips Research, The Netherlands, from 2005 to 2009. He is currently the CEO and Chief Scientist at Inception Institute of Artificial Intelligence, Abu Dhabi, UAE. His research interests include computer vision, image/video processing, and machine learning. He is a fellow of the British Computer Society and the Institution of Engineering and Technology. He is an Associate Editor of IEEE Transactions on Image Processing, IEEE Transactions on Neural Networks and Learning Systems, IEEE Transactions on Circuits and Systems for Video Technology, and several other journals.
\end{IEEEbiography}

\vspace{-10mm}
\begin{IEEEbiography}[{\includegraphics[width=1in,height=1.25in,clip,keepaspectratio]{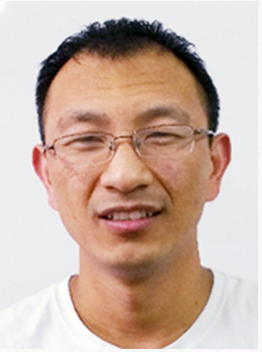}}]{Wangmeng Zuo} (M'09, SM'15) received the Ph.D. degree in computer application technology from the Harbin Institute of Technology, Harbin, China, in 2007. He is currently a Professor in the School of Computer Science and Technology, Harbin Institute of Technology. His current
research interests include image enhancement
and restoration, object detection, visual tracking,
and image classification. He has published over
100 papers in top-tier academic journals and conferences. He has served as a Tutorial Organizer
in ECCV 2016, an Associate Editor of the IET Biometrics and Journal of Electronic Imaging, and the Guest Editor of Neurocomputing, Pattern
Recognition, IEEE Transactions on Circuits and Systems for Video Technology, and IEEE Transactions on Neural Networks and Learning Systems. As of 2019, his publications have been cited more than 10,000 times in the literature. 
\end{IEEEbiography}

\vspace{-10mm}
\begin{IEEEbiography}[{\includegraphics[width=1in,height=1.25in,clip,keepaspectratio]{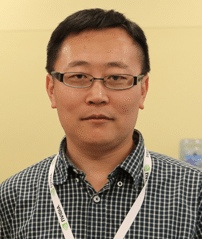}}]{Deyu Meng} received the B.S., M.S., and Ph.D. degrees from Xi'an Jiaotong University, Xian, China, in 2001, 2004, and 2008, respectively. From 2012 to 2014, he took his two-year sabbatical leave in Carnegie Mellon University. He is currently a Professor with the Institute for Information and System Sciences, School of Mathematics and Statistics, Xi'an Jiaotong University. His current research interests include
self-paced learning, noise modeling, and tensor sparsity.
\end{IEEEbiography}

\vspace{-10mm}
\begin{IEEEbiography}[{\includegraphics[width=1in,height=1.25in,clip,keepaspectratio]{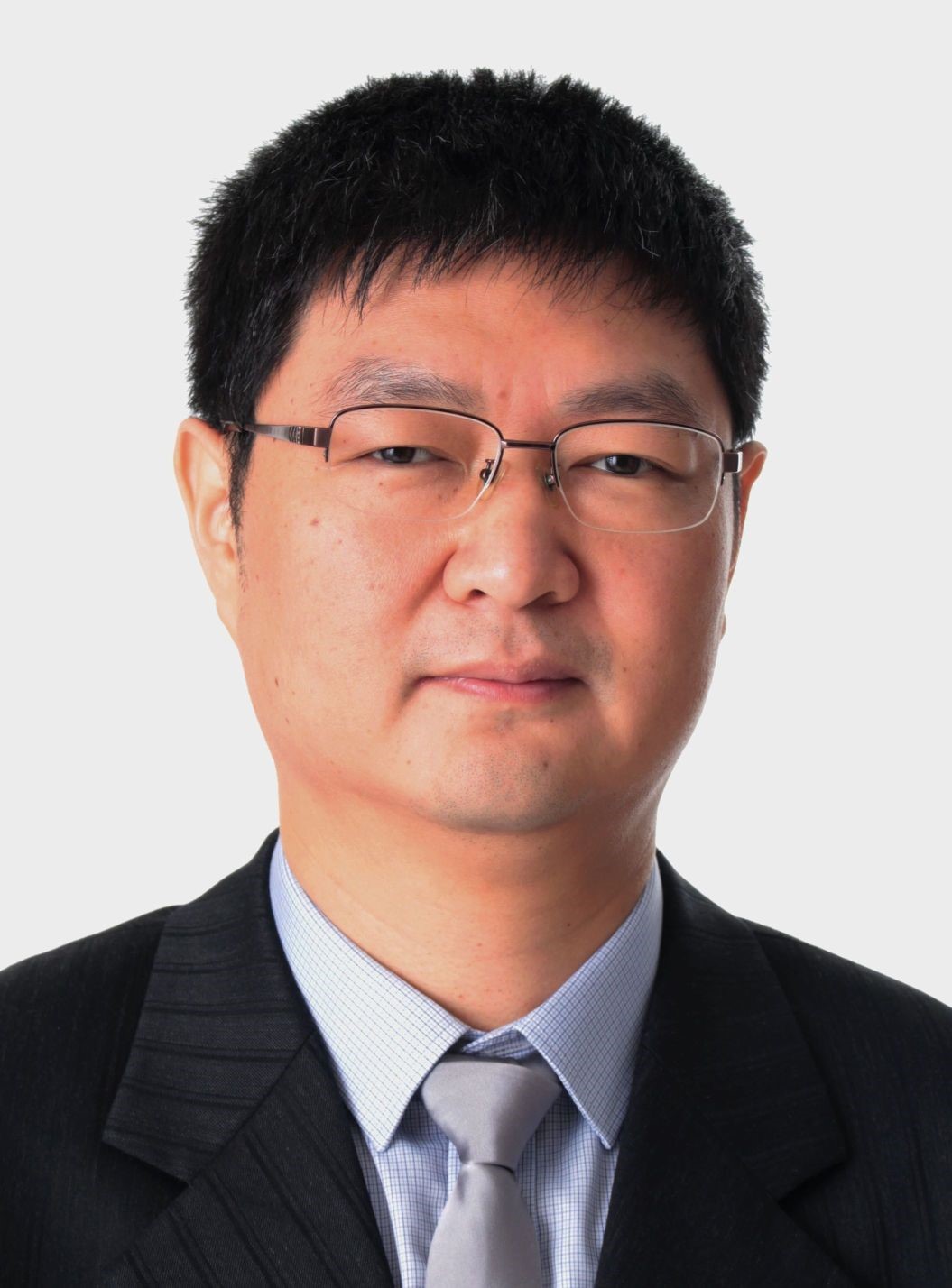}}]{Lei Zhang} (M'04, SM'14, F'18) received his B.Sc. degree in 1995 from Shenyang Institute of Aeronautical Engineering, Shenyang, P.R. China, and M.Sc. and Ph.D degrees in Control Theory and Engineering from Northwestern Polytechnical University, Xi'an, P.R. China, respectively in 1998 and 2001, respectively. From 2001 to 2002, he was a research associate in the Department of Computing, The Hong Kong Polytechnic University. From January 2003 to January 2006 he worked as a Postdoctoral Fellow in the Department of Electrical and Computer Engineering, McMaster University, Canada. In 2006, he joined the Department of Computing, The Hong Kong Polytechnic University, as an Assistant Professor. Since July 2017, he has been a Chair Professor in the same department. His research interests include Computer Vision, Pattern Recognition, Image and Video Analysis, and Biometrics, etc. Prof. Zhang has published more than 200 papers in those areas. As of 2019, his publications have been cited more than 43,000 times in the literature. Prof. Zhang is an Associate Editor of \textsl{IEEE Trans. on Image Processing}, \textsl{SIAM Journal of Imaging Sciences} and \textsl{Image and Vision Computing}, etc. He is a ``Clarivate Analytics Highly Cited Researcher'' from 2015 to 2018. More information can be found in his homepage \url{http://www4.comp.polyu.edu.hk/~cslzhang/}.
\end{IEEEbiography}

\vspace{-10mm}
\begin{IEEEbiography}[{\includegraphics[width=1in,height=1.25in,clip,keepaspectratio]{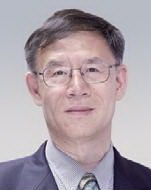}}]{David Zhang} (F'08) received the degree in computer science from Peking University, the M.Sc. degree in 1982, and the Ph.D. degree in computer science from the Harbin Institute of Technology (HIT), in 1985, respectively. From 1986 to 1988, he was a Post-Doctoral Fellow with Tsinghua University and an Associate Professor with the Academia Sinica, Beijing. In 1994, he received the second Ph.D. degree in electrical and computer engineering from the University of Waterloo, Ontario, Canada. He is currently the Chair Professor with the Hong Kong Polytechnic University, since 2005, where he is the Founding Director of the Biometrics Research Centre (UGC/CRC) supported by the Hong Kong SAR Government in 1998. He is a Croucher Senior Research Fellow, Distinguished Speaker of the IEEE Computer Society, and a Fellow of IAPR. So far, he has published over 20 monographs, over 400 international journal papers and over 40 patents from USA/Japan/HK/China. He was selected as a Highly Cited Researcher in Engineering by Thomson Reuters in 2014, 2015, and 2016, respectively. He also serves as a Visiting Chair Professor with Tsinghua University and an Adjunct Professor with Peking University, Shanghai Jiao Tong University, HIT, and the University of Waterloo. He is Founder and Editor-in-Chief, International Journal of Image and Graphics, the Founder and the Series Editor, Springer International Series on Biometrics (KISB); Organizer, International Conference on Biometrics Authentication, an Associate Editor for over ten international journals including the IEEE Transactions and so on.
\end{IEEEbiography}

\end{document}


\title{Supplementary file to
\\ ``Scaled Simplex Representation for Subspace Clustering''}

\author{
\IEEEauthorblockN{
Jun Xu$^{1,2}$,
Mengyang Yu$^{2}$,
Ling Shao$^{2}$,~\IEEEmembership{Senior Member,~IEEE},
Wangmeng Zuo$^{3}$,~\IEEEmembership{Senior Member,~IEEE},
\\
Deyu Meng$^{4}$,
Lei Zhang\textsuperscript{5},~\IEEEmembership{Fellow,~IEEE},
David Zhang\textsuperscript{5,6},~\IEEEmembership{Fellow,~IEEE}
}
\IEEEauthorblockA{
$^{1}$College of Computer Science, Nankai Univeristy, Tianjin, China
\\
$^{2}$Inception Institute of Artificial Intelligence, Abu Dhabi, UAE
\\
$^{3}$School of Computer Science and Technology, Harbin Institute of Technology, Harbin, China
\\
$^{4}$School of Mathematics and Statistics, Xi'an Jiaotong University, Xi'an, China
\\
$^{5}$Department of Computing,
The Hong Kong Polytechnic University, Hong Kong SAR, China
\\
$^{6}$School of Science and Engineering, The Chinese University of Hong Kong (Shenzhen), Shenzhen, China
}
\thanks{Corresponding author: Jun Xu (email: nankaimathxujun@gmail.com).}
}

\maketitle

\section{Solution of the NLSR Model}
\label{appendixa}
The NLSR model (Eqn.~(20) in the main paper) does not have an analytical solution. We employ a  variable splitting method \cite{courant1943,Eckstein1992} to solve it. By introducing an auxiliary variable $\bm{Z}$, we can reformulate the NLSR model into a linear equality-constraint problem with two variables $\bm{C}$ and $\bm{Z}$:
\begin{equation}
\begin{split}
\label{e-renlsr}
&
\min_{\bm{C},\bm{Z}}
\|
\bm{X}
-
\bm{X}\bm{C}
\|_{F}^{2}
+
\lambda
\|\bm{C}\|_{F}^{2}
\quad 
\text{s.t.}
\quad
\bm{Z}=\bm{C}
,
\bm{Z}\ge0
.
\end{split}
\end{equation}
Since the objective function is separable w.r.t. the variables $\bm{C}$ and $\bm{Z}$, problem (\ref{e-renlsr}) can be solved under the alternating direction method of multipliers (ADMM)~\cite{admm} framework. The Lagrangian function of the problem (\ref{e-renlsr}) is 
\begin{equation}
\begin{split}
\label{e-Lagnlsr}
\mathcal{L}(\bm{C},\bm{Z},\bm{\Delta},\lambda,\rho)
=
&
\|
\bm{X}
-
\bm{X}\bm{C}
\|_{F}^{2}
+
\lambda
\|\bm{C}\|_{F}^{2}
\\
&
+
\langle
\bm{\Delta},\bm{Z}-\bm{C}
\rangle
+
\frac{\rho}{2}
\|
\bm{Z}
-
\bm{C}
\|_{F}^{2}
,
\end{split}
\end{equation}
where $\bm{\Delta}$ is the augmented Lagrangian multiplier and $\rho>0$ is the penalty parameter.\ We initialize the vector variables $\bm{C}_{0}$, $\bm{Z}_{0}$, and $\bm{\Delta}_{0}$ to be conformable zero matrices and set $\rho>0$ with a suitable value.\ Denote by ($\bm{C}_{k}$, $\bm{Z}_{k}$) and $\bm{\delta}_{k}$ the optimization variables and the Lagrange multiplier at iteration $k$ ($k = 0, 1, 2,..., K$), respectively. The variables can be updated by taking derivatives of the Lagrangian function (\ref{e-Lagnlsr}) w.r.t. the variables $\bm{C}$ and $\bm{Z}$ and setting them to be zero.
\vspace{1mm}
\\
(1) \textbf{Updating $\bm{C}$ while fixing $\bm{Z}$ and $\bm{\Delta}$}:
\begin{equation}
\begin{split}
\label{e-nlsrC}
&
\min_{\bm{C}}
\|
\bm{X}
-
\bm{X}\bm{C}
\|_{F}^{2}
+
\lambda
\|\bm{C}\|_{F}^{2}
+
\frac{\rho}{2}
\|
\bm{C}
-
(
\bm{Z}_{k}
+
\rho^{-1}
\bm{\Delta}_{k}
)
\|_{F}^{2}
.
\end{split}
\end{equation}
This is a standard least squares regression problem with closed form solution:
\begin{equation}
\begin{split}
\label{e-solvenlsrC}
\bm{C}_{k+1}
=
(\bm{X}^{\top}\bm{X}+\frac{2\lambda+\rho}{2}\bm{I})^{-1}
(\bm{X}^{\top}\bm{X}+\frac{\rho}{2}\bm{Z}_{k}+\frac{1}{2}\bm{\Delta}_{k})
\end{split}
\end{equation}
\\
(2) \textbf{Updating $\bm{Z}$ while fixing $\bm{C}$ and $\bm{\Delta}$}:
\begin{equation}
\begin{split}
\label{e-nlsrZ}
&
\min_{\bm{Z}}
\|
\bm{Z}
-
(
\bm{C}_{k+1}
-
\rho^{-1}
\bm{\Delta}_{k}
)
\|_{F}^{2}
\quad 
\text{s.t.}
\quad
\bm{Z}\ge0
.
\end{split}
\end{equation}
The solution of $\bm{Z}$ is
\begin{equation}
\begin{split}
\label{e-solvenlsrZ}
\bm{Z}_{k+1}=\max(0,\bm{C}_{k+1}-\rho^{-1}\bm{\Delta}_{k}),
\end{split}
\end{equation}
where the ``$\max(\cdot)$'' operator outputs element-wisely the maximal value of the inputs.
\vspace{1mm}
\\
(3) \textbf{Updating the Lagrangian multiplier $\bm{\Delta}$}:
\begin{equation}
\begin{split}
\label{e27}
\bm{\Delta}_{k+1}
&
=
\bm{\Delta}_{k}
+
\rho
(\bm{Z}_{k+1}-\bm{C}_{k+1})
.
\end{split}
\end{equation}

The above alternative updating steps are repeated until the convergence condition is satisfied or the number of iterations exceeds a preset threshold $K$.\ The convergence condition of the ADMM algorithm is: $\|\bm{Z}_{k+1}-\bm{C}_{k+1}\|_{F}\le \text{Tol}$, $\|\bm{C}_{k+1}-\bm{C}_{k}\|_{F}\le \text{Tol}$, and $\|\bm{Z}_{k+1}-\bm{Z}_{k}\|_{F}\le \text{Tol}$ are simultaneously satisfied, where $\text{Tol}>0$ is a small tolerance value. Since the objective function and constraints are all strictly convex, the NLSR model solved by the ADMM algorithm~\cite{admm} is guaranteed to converge to a global optimal solution.

\section{Solution of the SLSR Model}
\label{appendixb}
We solve the SLSR model (Eqn.~(21) in the main paper) by employing variable splitting methods~\cite{courant1943,Eckstein1992}. Specifically, we introduce an auxiliary variable $\bm{Z}$ into the SLSR model, which can then be equivalently reformulated as a linear equality-constrained problem:
\begin{equation}
\begin{split}
\label{e30}
\min_{\bm{C},\bm{Z}}
&
\|
\bm{X}
-
\bm{X}\bm{C}
\|_{F}^{2}
+
\lambda
\|
\bm{Z}
\|
_{F}^{2}
\\
\text{s.t.}
\quad
&
\bm{1}^{\top}\bm{Z}=s\bm{1}^{\top}
,
\bm{Z}
=
\bm{C}
,
\end{split}
\end{equation}
whose solution for $\bm{C}$ coincides with the solution of Eqn.~(20) in the main paper. Since its objective function is separable w.r.t. the variables $\bm{C}$ and $\bm{Z}$, problem (\ref{e30}) can also be solved via the ADMM method~\cite{admm}. The corresponding augmented Lagrangian function is the same as in Eqn.~(11) in the main paper.\ Denote by ($\bm{C}_{k}, \bm{Z}_{k}$) and $\bm{\Delta}_{k}$ the optimization variables and Lagrange multiplier at iteration $k$ ($k=0,1,2,...$), respectively. We initialize the variables $\bm{C}_{0}$, $\bm{Z}_{0}$, and $\bm{\Delta}_{0}$ to be conformable zero matrices. By taking derivatives of the Lagrangian function $\mathcal{L}$ (Eqn.~(11) in the main paper) w.r.t. $\bm{C}$ and $\bm{Z}$, and setting them to be zeros, we can alternatively update the variables as follows:
\vspace{1mm}
\\
(1) \textbf{Updating $\bm{C}$ while fixing $\bm{Z}_{k}$ and $\bm{\Delta}_{k}$}:
\begin{equation}
\begin{split}
\label{e31}
\bm{C}_{k+1}
=
\arg\min_{\bm{C}}
\|
\bm{X}
-
\bm{X}\bm{C}
\|
_{F}^{2}
+
\frac{\rho}{2}
\|
\bm{C}
-
(
\bm{Z}_k
+
\frac{1}{\rho}
\bm{\Delta}_k
)
\|
_{F}^{2}
.
\end{split}
\end{equation}
This is a standard least square regression problem and has a closed-from solution given by
\begin{equation}
\begin{split}
\label{e32}
&
\hspace{-3mm}
\bm{C}_{k+1}
=
(\bm{X}^{\top}\bm{X}+\frac{\rho}{2}\bm{I})^{-1}
(\bm{X}^{\top}\bm{X}+\frac{\rho}{2}\bm{Z}_{k}+\frac{1}{2}\bm{\Delta}_{k})
.
\end{split}
\end{equation}
\\
(2) \textbf{Updating $\bm{Z}$ while fixing $\bm{C}_{k}$ and $\bm{\Delta}_{k}$}:
\begin{equation}
\begin{split}
\label{e33}
\bm{Z}_{k+1}
&
=
\arg\min_{\bm{Z}}
\|
\bm{Z}
-
\frac{\rho}{2\lambda+\rho}
(
\bm{C}_{k+1}-\rho^{-1}\bm{\Delta}_{k}
)
\|_{F}^{2}
\\
&
\quad
\text{s.t.}
\quad 
\bm{1}^{\top}\bm{Z}=s\bm{1}^{\top}
.
\end{split}
\end{equation}
This is a quadratic programming problem and the objective function is strictly convex, with a close and convex constraint, so there is a unique solution. Here, we employ the projection based method~\cite{duchi2008efficient}, whose computational complexity is $\mathcal{O}(N\log{N})$ to process a vector of length $N$. Denote by $\bm{v}_{k+1}$ an arbitrary column of $\frac{\rho}{2\lambda+\rho}(\bm{C}_{k+1}-\rho^{-1}\Delta_k)$, the solution of $\bm{z}_{k+1}$ (the corresponding column in $\bm{Z}_{k+1}$) can be solved by projecting $\bm{v}_{k+1}$ onto a scaled affine space~\cite{duchi2008efficient}. The solution of problem (\ref{e33}) is summarized in Algorithm 4.
\begin{table}[t]
\centering
\begin{tabular}{l}
\Xhline{1pt}
\textbf{Algorithm 4}: Projection of the vector $\bm{v}_{k+1}$ onto a scaled affine space
\\
\hline
\textbf{Input:} Data point $\bm{v}_{k+1}\in\mathbb{R}^{N}$, scalar $s$.
\\
1. Sort $\bm{v}_{k+1}$ into $\bm{w}$: $w_1\ge w_2\ge ...\ge w_N$;
\\
2. Find $\alpha=\max\{1\le j\le N: w_j+\frac{1}{j}(s-\sum_{i=1}^{j}w_i)>0\}$;
\\
3. Define $\beta=\frac{1}{\alpha}(s-\sum_{i=1}^{\alpha}w_i)$;
\\
\textbf{Output:} $\bm{z}_{k+1}$: $z_{k+1}^i=v_{k+1}^i+\beta$, $i=1,...,N$. 
\\
\Xhline{1pt}
\end{tabular}
\label{a1}
\end{table}
\vspace{1mm}
\\
(3) \textbf{Updating $\bm{\Delta}$ while fixing $\bm{C}_{k}$ and $\bm{Z}_{k}$}:
\begin{equation}
\begin{split}
\label{e32}
\bm{\Delta}_{k+1}
&
=
\bm{\Delta}_{k}
+
\rho
(\bm{Z}_{k+1}-\bm{C}_{k+1})
.
\end{split}
\end{equation}

We repeat the above alternative updates until a certain convergence condition is satisfied or the number of iterations reaches a preset threshold $K$.\ The convergence condition of the ADMM algorithm is met when $\|\bm{C}_{k+1}-\bm{Z}_{k+1}\|_{F}\le \text{Tol}$, $\|\bm{C}_{k+1}-\bm{C}_{k}\|_{F}\le \text{Tol}$, and $\|\bm{Z}_{k+1}-\bm{Z}_{k}\|_{F}\le \text{Tol}$ are simultaneously satisfied, where $\text{Tol}>0$ is a small tolerance value.\ Since the objective function and constraints are convex, the SLSR model solved by the ADMM algorithm, is guaranteed to converge to a global optimal solution.

{
\small
\balance
\bibliographystyle{unsrt}
\bibliography{SSRSC}
}